\documentclass[envcountsect,runningheads]{llncs}
\usepackage[T1]{fontenc}
\usepackage{graphicx}
\usepackage{amsfonts,iftex,listings,fancyvrb,longtable,bookmark,xcolor,soul,setspace,babel}
\usepackage{latexsym}
\usepackage{amssymb}
\usepackage{amsmath}
\usepackage{caption}
\usepackage{booktabs}
\usepackage{enumitem}
\usepackage{graphicx}
\usepackage{color}
\usepackage{multirow}
\usepackage{multicol}
\usepackage{times}
\usepackage{fancyhdr}
\usepackage{amssymb}
\usepackage{caption}
\usepackage{subcaption}
\usepackage[ruled,vlined]{algorithm2e}

\newcommand{\probP}{\text{I\kern-0.15em P}} 
\newcommand{\indep}{\perp \!\!\! \perp} 

\spnewtheorem{thm}{Theorem}[section]{\bfseries}{\itshape}
\spnewtheorem{prop}{Proposition}[section]{\bfseries}{\itshape}
\spnewtheorem{fact}{Fact}[section]{\bfseries}{\itshape}

\begin{document}
\title{ {\huge Working Paper:} \\ Active Causal Structure Learning with Latent Variables: Towards Learning to Detour in Autonomous Robots}
\titlerunning{Working Paper: ACSLWL}
%
\author{Pablo de los Riscos\inst{1} \and
Fernando Corbacho\inst{1,2}}
\authorrunning{P. Riscos and F. Corbacho}
%
\institute{Cognodata R\&D, Paseo de la Castellana, 135 28046 Madrid, Spain \and Computer Engineering Department, Universidad Autónoma de Madrid, 28049 Madrid, Spain}
\maketitle              
%
\begin{abstract}
Artificial General Intelligence (AGI) Agents and Robots must be able to cope with everchanging environments and tasks. They must be able to actively  construct new internal causal models of their interactions with the environment when new structural changes take place in the environment. Thus, we claim that active causal structure learning with latent variables (ACSLWL) is a necessary component to build AGI agents and robots. This paper describes how a complex planning and expectation-based detour behavior can be learned by ACSLWL when, unexpectedly, and for the first time, the simulated robot encounters a sort of  “transparent” barrier in its pathway towards its target. ACSWL consists of acting in the environment, discovering new causal relations, constructing new causal models, exploiting the causal models to maximize its expected utility,  detecting possible latent variables when unexpected observations occur, and constructing new structures – internal causal models and optimal estimation of the associated parameters, to be able to cope efficiently with the new encountered situations. That is, the agent must be able to construct new causal internal models that transform a previously unexpected and inefficient (sub-optimal) situation,  into a predictable situation with an optimal operating plan.

\end{abstract}
\section{Introduction}

In this paper, we start with an agent whose brain´s dynamic decision network (DDN) has been learned in an environment in which “transparent” barriers were not previously present \footnote{Transparent barrier: pailing fence with wide enough gaps that allow to see through, yet it does not allow the robot to pass through.}. Then, suddenly,  a sort of “transparent” barrier is introduced in the environment, on the pathway towards the agent´s target,  so that this becomes a new situation that the agent, initially, does not know how to efficiently handle. Then, the agent interacts within the environment and new unexpected observations take place. In order to cope with this new situation, the agent responds by generating a new hidden variable, that is causally related in the dynamic decision network with certain decision and chance variables. It also estimates the new conditional probability distributions related with this hidden variable. That is, the agent´s objective is to learn to adapt its knowledge (e.g. the transition probabilities) acquired during this learning process, so that, it will be able to later predict the effects of its actions in this new environmental situation. \\

We propose a new active causal structure learning with latent variables (ACSLWL) framework that incorporates several aspects. First, we introduce a new coefficient of surprise, for the agent to be able to recognize unprecedented unexpected situations. Then, we propose latent variable detection when a significant surprise in the utility function takes place.  It follows with hidden variable sub-graph structure learning to adapt the structure of the DDN graph to the new detected latent variable and its topological structure with respect to the already existing chance variables in the DDN.  The selection of candidates for the hidden variable´s children and parents is also based on the relevance of each particular chance variable on the coefficient of surprise for the overall probabilistic graph joint distribution. For this paper, we propose a “XM” topological structure that we will define later. Once, this new structure is in place, then,  hidden variable sub-graph parametric learning takes place to adjust the CPT parameters for its children nodes, as well as, the CPT parameters of the hidden variable depending on its parent nodes. In this paper, we use hard weighted EM for parameter estimation.

\subsection{Related Work}
The main thrust of the work presented in this paper is related to the topic of causality in machine learning \cite{Scholkopf22,Schölkopf2022b,Sch_lkopf_2019,Schölkopf2022} and, more specifically,  structure learning of causal models in the presence of latent variables, as it will be explained below. It is also related to the learning of behaviors in animals and autonomous robots \cite{Arbib2002,Arkin2005,Arkin1998a,Bekey2005,Corbacho2005,Corbacho1995,Kober2013,Nguyen2011}. 
Thus, our work is related, among other things,  to structure learning in Bayesian networks \cite{Kitson2023,villa-Blanco2022}, causal structure learning, structure learning with latent variables, and active structure learning.  Causal structure learning (CSL) \cite{Heinze-Deml2018} aims at learning the causal graph (where edges are interpreted as direct causal effects) that best describes the dependence and causal structure in a given data set. Causal models allow to make predictions under arbitrary interventions. CSL, in general, takes into account different possible interventions, such as do-interventions and additive interventions.  In this paper we focus on do-interventions that are selected by MEU. Causal sufficiency refers to the absence of latent variables \cite{Spirtes2001}. Yet, when possible latent variables are considered, there are two common options for the modeling of latent variables. They can be modeled explicitly as nodes in the structural equations, or they can manifest themselves as a dependence between the noise terms, where the noise terms are assumed to be independent in the absence of latent confounding \cite{Heinze-Deml2018}. In this paper, we chose the first option, that is, the latent variables are modeled explicitly by the introduction of new nodes in the graph. We define  these new nodes in the causal graph as hidden variables from the agent´s perspective, since they never get to be observed.\\ 

Elidan et al.\cite{Elidan2001,Elidan2007} describe structure learning with hidden variables. Elidan et al.\cite{Elidan2001}  discover hidden variables by a structure-based approach based on the detection of cliques. On the other hand,  Elidan et al.\cite{Elidan2007} identify hidden variables thanks to the ideal parent paradigm. First, they decide the existence of a new hidden variable by a subset (cluster) of variables that share a similar profile. Then, the children of the new hidden variable correspond to the variables in the cluster. Squires et al.\cite{squires2022} also learn the structure of a causal model in the presence of latent variables. They introduce latent factor causal models (LFCMs) and identify clusters of observed variables. Then, they merge clusters when necessary, and finally, they learn the edges from the observed variables to the latent variables. 
On the other hand, Sontakke et al. \cite{sontakke2021} propose causal structure learning with latent variables by learning optimal sequences of actions and discovering causal factors. They describe (directly unobserved) causal factors and formalize them by latent variables learned thanks to the causal curiosity paradigm. Transition functions are affected by a subset of the global causal factors. In their work, Causal POMDPs extend POMDP by explicitly modelling the effect of causal factors on observations. The uncontrollable portion of the state consists of the causal factors of the environment. They utilize the outcome of performing the experimental behaviors to infer a representation for the causal factor isolated by  the experiment in question. Their method, isolates the individual causal factors, which they assume are independent by algorithmic information theory implemented by MDL. In their framework, causal curiosity corresponds to an inherent reward that helps solving the optimization problem.\\

Tong \& Koller \cite{Tong2001b} describe active learning for structure in Bayesian Networks. That is, learning the causal structure of a domain more efficiently by selecting  the sequence of optimal queries. They define a (myopic) active learner $l$ as a function that selects a query $\mathbf{Q} := \mathbf{q}$ based upon its current distribution over $G$ (the DAG) and $\mathbf{\Theta}_G$ (the associated parameters)\cite{Tong2001b}. They define a measure for the quality of the distribution over graphs and parameters  and use it to evaluate the extent to which various instances would improve the quality of the distribution, thereby, providing a means to select the next query to perform. Given a query they define the expected posterior loss of the query based on the edge entropy loss function and minimize the Expected posterior Loss (MEL).

\section{Technical preliminaries}
In this section, we introduce some concepts that will set the theoretical background for the rest of the work.

\subsection{Partially Observed Markov Decision Process (POMDP)}
A Markov decision processes (MDP) is a mathematical model used to describe the interactions between an agent and its environment. The model is based on the assumption that the agent is able to fully observe the environment, which allows the agent to be aware of its exact state at all times. Furthermore, when combined with the Markov assumption on the transition model, it ensures that optimal policies are always dependent on the current state.
Yet, in general, agents are not able to completely observe the environment. Rather, they observe it partially, which means that the agent will not know exactly the state in which it is at any given moment. This aspect makes the problem considerably more complex to solve, as the actions indicated by the optimal policy will no longer necessarily be the most beneficial for the agent, and the utilities will not necessarily be the most optimal for the agent. Moreover, the utilities of each state (s) and the optimal actions taken in response to each state will no longer depend solely on the state itself, but also on the degree to which the agent believes they are in that state.
This introduces the concept of a Partially Observable Markov Decision Process (POMDP), which models the interactions of an agent with an environment that is partially observable. This formal concept shares the same characteristics as a Markov Decision Process (MDP), but with the addition of two elements to address the issue of partial observation of the environment.
\begin{definition}
    A time-discrete POMDP is formally defined as a 7-tuple $(S,A,T,R,\Omega,O,\gamma)$, where:
\begin{itemize}
    \item $S$ is the set of states that represent all possible configurations of the environment.
    \vspace{2mm}
    \item $A$ is the set of actions that the agent can take.
    \vspace{2mm}
    \item $T$ models the conditional probability distribution of transitioning to state $s^\prime$ from state $s$ after the agent performs action $a$, $\probP(s^\prime |s,a)$.
    \vspace{2mm}
    \item $R \colon (A,S) \rightarrow \mathbb{R}$ is the utility function, which states the reward the agent receives for executing action $a$ and being in state $s$
    \vspace{2mm}
    \item $\Omega$ is the set of possible observations that the agent can perceive from the states of the environment.
    \vspace{2mm}
    \item $O$ models the conditional probabilities of observing $o$ given that action $a$ was taken and the resulting state is $s^\prime$,  $\probP(o|s^\prime,a)$.
    \vspace{2mm}
    \item $\gamma \in [0,1)$ is a factor that discounts future rewards, reflecting the preference for immediate rewards over future ones.
\end{itemize}
\end{definition}

\subsection{Dynamic Decision Networks (DDN)}
Just as it is necessary to work within a framework that allows modeling the interactions of the agent with the environment, it is also necessary to have a tool that allows representing and working with POMDP in a simpler and structured way. In this work, we will use dynamic decision networks, a probabilistic graphical model that extends decision networks to deal with the temporal relationships between observation and action variables and to come up with optimal policies.\\

\begin{definition}
    A dynamic influence diagram (DID) or dynamic decision network (DDN) is a probabilistic graphical model consisting of a tuple $\mathcal{DDN}=(G,P)$, where $G$ is a directed acyclic graph (DAG) and $P$ the parameters modeling the joint probability distribution determined by the relationships given in $G$.\\

    The DAG $G=(Z,E)$ has three different types of nodes:
    \begin{itemize}
        \item Chance nodes, which represent the random variables in the model. Each chance node has an associated probability distribution  $\probP(X |Pa(X))$ where $Pa(X)$ is the set of parents of $X$ in the graph . The set of chance nodes is denoted by $\mathcal{X}$.
        \vspace{1mm}
        \item Decision nodes, which represent the action variables and whose values are chosen by the agent. The set of decision nodes is denoted by $D$. 
        \vspace{1mm}
        \item Utility nodes, which represent the agent´s utility functions. These nodes will have no children in the graph $G$, and their parents will be those chance or decision nodes that influence the value of the utility function. The set of utility nodes is denoted by $\mathcal{U}$.
        
    \end{itemize}
    Furthermore, the DDN $G$ can be divided into two different DNs:
    \begin{enumerate}
        \item A decision network $DN_0=(G_0,P_0)$, that models the relations between variables in a fixed point in time and the underlying joint probability distribution.
        \item A decision network $DN_{\rightarrow}=(G_{\rightarrow},P_{\rightarrow})$, showing the temporal relationships between the variables and the transition probability distribution \\
        $P(X_{1;t},\dots,X_{n;t}|X_{1;t-1},\dots,X_{n;t-1}) $.  
\end{enumerate}
    
    It is assumed that the changes occur between discrete time moments, which will be indexed by a non-negative integer. The index representing time will be denoted by $t \; \in \{1,\dots ,T\}$ and the notation $X_{i;t}$ will be used for the value of the variable $X_i$ at time $t$.
\end{definition}

It is important to note that there is an important limitation when representing a POMDP with DDNs, and it is that DDNs usually consider a finite number of time moments, so that in reality the problem being represented is a finite horizon POMDP.\\

When calculating the optimal values of decision nodes, given a set of evidence over the chance nodes, this framework follows the principle of Maximum Expected Utility (MEU). This principle is based on the fact that an agent is rational when it takes those actions that produce the highest expected utility, averaged over all possible values of the actions. Therefore, our agent will always choose to take those actions that have the highest expected utility, given some evidence on the chance nodes.
The reason we choose to work with DDNs is that this tool allows us to check, after a decision has been made, whether the values of the chance nodes in the following time moment are the expected ones or they are not, and let the agent act accordingly.

\subsection{Causality}
\subsubsection{Causal relations between random variables}

The term ``causality'' has been defined in various ways, and there is currently no consensus on a single unified definition. In this work, we will examine the relationship between two random variables, $X$ and $Y$, which may be binary, categorical, or continuous.
The three most commonly used definitions for causality are as follows:
\begin{itemize}
    \item \textbf{Counterfactual Causality:} $X$ is a cause of $Y$ if specific values of $X$ and $Y$ were observed ($X=x$, $Y=y$), yet if a different value of $X$ had been observed, ($X=x^\prime$), then the value of $Y$ would be different, ($Y=y^\prime$).
    \vspace{2.5mm}
    \item \textbf{Interventional Causality:} $X$ is a cause of $Y$ if, by making the variable $X$ take a certain value, $X=do(x)$, then we have that:
    \[P(Y|X=do(x))=P(Y|X=x),\]
    whereas this is not the case for $Y$:
    \[P(X|Y=do(y)) \neq P(X|Y=y).\]
    \item \textbf{Mechanistic Causality:} $X$ is a cause of $Y$ if there exists a function $f$ such that $Y:=f(X,\mathcal{N})$, that enables the generation of $Y$ values based on $X$ values and a random noise $\mathcal{N}$.
\end{itemize}

\noindent These three definitions are related but not entirely equivalent. This work will focus on the definition of mechanistic causality, thus, it will be assumed that underlying mechanistic functions exist whenever two random variables are causally related.
Based on the mechanistic causality definition, all possible causal relations between a pair of random variables can be classified as follows:

\begin{definition} Causal relations between a pair of random variables.\\
Given a pair of random variables $X$ and $Y$, there are four potential relationships between them:
    \[\left\{ \begin{array}{l}
    R=X \xrightarrow{} Y \Rightarrow  \exists \;f  \wedge \exists \; \mathcal{N}_y \;\; | \;   Y:=f(X,\mathcal{N}_y) \\ \\ 
    
    R=X \leftarrow{} Y \Rightarrow  \exists \;f  \wedge \exists \; \mathcal{N}_x \;\; | \;   X:=f(Y,\mathcal{N}_x)  \\ \\ 
     
    R=X \leftrightarrow Y \Rightarrow  \exists \;f,g  \wedge \exists \; \mathcal{N}_x,\mathcal{N}_y,Z \;\; | \;   X:=f(Z,\mathcal{N}_x) \wedge Y:=g(Z,\mathcal{N}_y)\\ \\
    
    R=X \perp Y \Rightarrow  X \text{ and } Y \text{are independent, there is no function that relates them to each other }
    
    \end{array} \right.\]

\end{definition}

\subsubsection{Causal relations between actions and random variables}
\noindent In the previous section, we described the causal relationship between two random variables. However, when we have a random variable and an action variable, it is necessary to qualify the concept of causality. To study this causal relationship, we will assume that the only possibilities for a random variable, $O$, and an action variable, $D$, are:
\[\left\{ \begin{array}{l}
    R=D \xrightarrow{} O \Rightarrow  \text{The activation function of $D$ influences the values that $O$ takes} \\ \\ 
    
    R=D \perp O \Rightarrow \text{The activation function of $D$ does not influence the values that $O$ takes}
    
\end{array} \right.\]

The reason why the relationships $O \xrightarrow{} D$ and $O \leftrightarrow D$  will not be considered is because this work will focus on the causality from the point of view of an agent. In this context, the action variables are not random, instead, they are the result of optimizing a utility function from the perspective of the agent. In other words, the cause of an action variable taking a value can be attributed to the agent's own behavior in response to the random values observed and the utility function being optimized and not the values of the random values itself.
However, from an external point of view outside the agent, both relationships $O \xrightarrow{} D$ and $O \leftrightarrow D$ could also be considered. In this situation, the utility function and the maximization principle of utility is unknown, making the agent's behaviour random a priori. These relations would outline the utility function of the agent and how it selects the values of the actions to maximize the utility.

\subsubsection{Causal relations through time}
Finally, we will work with random variables whose values changes over time, so it is also important to discuss how time is involved in the causal relations between variables.
In the case of two random processes, $X_1$ and $X_2$, the Granger causality concept is often used. A stochastic process $X_1$ is a G-cause of the stochastic process $X_2$ if past values of $X_1$, in addition to past values of $X_2$, contain information that helps predict current or future values of $X_2$.
This definition holds whenever the following assumptions are met:
\begin{enumerate}
    \item Causes always occur before the effects.
    \item Causes contain unique information about the values that the effects will take in the future.
\end{enumerate}

In the case of temporal causal relationships between actions and random variables, we will also use Granger causality, taking into account the same detail of the agent's point of view, so that the only possible relationships are $Act \xrightarrow{G-cause}Obs$ or not. Note that in this context the two assumptions of Granger causality are also satisfied.

\subsection{Structure Learning}
One of the main problems in using Bayesian networks or Decision networks is that in real problems the structure and parameters of the network modeling the data are usually not available. This leads to the need to rely on expert knowledge to build and use these models.
This has led to two main areas, parametric learning and structural learning from a data set. This section will focus on describing the problem of structural learning and the difficulties that need to be faced in order to provide a solution.

\subsubsection{Problem definition and challenges}
Given a set of variables $\mathcal{X}$, and a set of fully observed data $\mathcal{D}$, the goal is to learn the graph $\mathcal{G}^*$ that models the conditional relations between the variables, corresponding to the joint probability distribution $\probP(\mathcal{X})$ underlying $\mathcal{D}$.\\

First, note that the structural learning problem is NP-hard. This is because the size of the space of possible DAGs of $n$ nodes, $|\mathcal{G}_n|$, is given by the following recursion:

\begin{equation}
    \begin{split}
        |\mathcal{G}_n|=\sum_i^n (-1&)^{i+1}\binom{n}{i}2^{i(n-i)} |\mathcal{G}_{n-i}| \;\;\;\; n>2,\\
        |&\mathcal{G}_0|=|\mathcal{G}_1|=1,
    \end{split}
\end{equation}

that grows super-exponentially with the number of nodes.
On the other hand, in order to understand another one of the challenges faced by structural learning, two definitions must first be taken into account:
\begin{definition}
    $I$-equivalence\\
    Two DAGs \textit{H}, $K$ are $I$-equivalent if they have the same set of nodes $\mathcal{V}$ and represent the same conditional independence relations, i.e 
    \begin{equation}
        (\mathcal{X} \indep \mathcal{Y} | \mathcal{Z})_H \Leftrightarrow (\mathcal{X} \indep \mathcal{Y} | \mathcal{Z})_K \;\;\;\;\; \forall \mathcal{X},\mathcal{Y},\mathcal{Z}\subset \mathcal{V}
    \end{equation}
\end{definition}
\begin{definition}
    $D$-map, $I$-map and perfect map\\
    Considering a joint probability distribution $P$ and a graph $G$ it is said that:
    \begin{enumerate}
        \item $G$ is a $D$-map of $P$ if all conditional independence relations of $P$ are shown in $G$, but all of $G$ do not necessarily occur in $P$.
        \vspace{1mm}
        \item $G$ is an $I$-map of $P$ if all conditional independence relations shown in $G$ occur in $P$, but all those in $P$ are not necessarily shown in $P$.
        \vspace{1mm}
        \item $G$ is a perfect map of $P$ if $G$ is a $D$-map and $I$-map of $P$ at the same time, meaning that all relations of $P$ are shown in $G$, and all relations of $G$ are given in $P$.
    \end{enumerate}
\end{definition}

These definitions state that we can only aspire to learn a graph $\hat{\mathcal{G}}$ that is I-equivalent to the real $\mathcal{G}^*$, since if $\mathcal{G}^*$ is a perfect map of $\probP(\mathcal{X})$, then all graphs I-equivalent to $\mathcal{G}^*$ will also be perfect maps of $\probP(\mathcal{X})$.
Therefore, passive structural learning, in general, does not have an unique solution. This is referred to as $\mathcal{G}$ not being identifiable from the data set $\mathcal{D}$.

\subsubsection{Structural learning with latent variables}\label{sect:Aprendizaje con vars latentes}
As we have already seen, the issue of structural learning with fully observed data is inherently intricate. The problem becomes even more complex when structural learning with latent variables is considered. These are some factors that contribute to this complexity:

\begin{enumerate}
    \item \textbf{Detection of latent variables:} As their name suggests, latent variables are those that cannot be directly observed. Consequently, it is essential to identify the number of latent variables that may be influencing the observed ones, and to what degree they influence them.
    \vspace{1mm}
    \item \textbf{Cardinality of latent variables:} Since they cannot be observed, the cardinality or domain of the latent variable is also unknown. In general, it is typically assumed that latent variables are discrete and their cardinality is determined by algorithms based on complexity reduction or likelihood maximization in the data set.
    \vspace{1mm}
    \item \textbf{Introduction of latent variables in the structure:} Assuming that we know how to detect the latent variables and how many are relevant, the next step, apart from discovering their cardinality, would be to learn how they are related to the rest of the observed variables, and even among the latent variables. This is necessary in order to introduce these variables in the structure that represents the conditional dependence relationships in an appropriate way.
\end{enumerate}

\begin{fact}
    In the literature, the terms latent variable and hidden variable are often mixed, referring to the concept of unobserved variable. However, in this paper, the difference between the concept of latent variable and hidden variable is as follows:
    \begin{itemize}
        \item A \textbf{latent variable} $LV$, is a variable unobserved by the agent and placed naturally in the environment, whether or not it is relevant to the problem faced by the agent.
        \vspace{1mm}
        \item A \textbf{hidden variable} $HV$, is considered to be a representation of one or several latent variables generated by the agent with the intention of modeling the latent variables. Since the latent variables cannot be fully observed, the agent needs to model or estimate them. Generally, the agent will only want to model those latent variables that are relevant to find a better solution to the problem the agent faces.
    \end{itemize}
\end{fact}
\subsubsection{Causal structure learning}
In addition to Bayesian Networks (BN) and Decision Networks (DN), there are extensions to these models, such as Causal Bayesian Networks (CBN) and Causal Decision Networks (CDN), which enable the representation of causal relations instead of association relations. The difference between these probabilistic models lies in the assumption that directed edges are causal. This implies that the structure of the model can only be represented by a single DAG structure. In contrast, a BN or DN can be represented by any DAG structure that belongs to its corresponding Markov equivalence class of DAG structures, as discussed at the beginning of this section.
At first glance, this could be thought of as an advantage, because this partially solves the identifiability problem, offering an unique solution for the problem (in the real world, there is only one way in which variables are causally related). However, the study of causal relationships is currently a relatively underdeveloped and challenging field of study, making the problem even more complex than structure learning of BNs or DNs \cite{Glymour2019,Peters2017,Spirtes2001}.  \\

The principal motivation to deal with this problem is that CDN, apart from been able to do predictive and diagnostic inference, are also capable of performing interventional and counterfactual inference. An agent that works with a CDN could address the three levels of the ladder of causation \cite{Pearl2000}. The first step of the ladder focuses on what an agent can learn from the associations alone; the second step focuses on the simulation of hypothetical interventions to measure their effect without the need to perform experiments; and the third step focuses on answering counterfactual questions about alternative actions that could have been taken in the past. These three steps, together with active learning, are fundamental for an AGI agent, as they will allow him to be even more autonomous and explanatory.
In this work we will work with Causal Dynamic Decision Networks (CDDN), an extension of Dynamic Decision Networks (DDN) in which the arcs represent causal relationships. However, throughout the paper, we will refer to CDDNs as DDN.

\section{Environment and Agent Formalization}

\subsection{Environmental selection motivation}
The problem we are going to face is the following: we are going to put a robot that has been trained to reach a target into a new environment where there is a barrier formed by spikes that prevents it from moving towards the target, but still allows it to observe it.
The barrier formed by spikes will be unknown to the agent, since it has been trained in an environment where the barrier did not exist. Therefore, the latent variable will be the presence or absence of areas where the agent cannot pass through the barrier. To avoid complications, the distance between the spikes will be equal and smaller than the width of the agent, so that the agent cannot pass between them.
Thus, this problem contains all the essential elements to test our framework, since the appearance of the barrier will cause the agent to change its behavior in order to reach the goal. In the framework proposed, this is done by an internal structural change.
The motivation for this environment comes from the study of Corbacho et al \cite{Corbacho2005,Corbacho1995}, where they experimented with frogs that were given a prey inside an enclosure and between them a barrier of sticks separated in such a way that the frogs could not pass but could see the prey. At the beginning of the experiment, the frogs collided with the barrier and tried to go through the sticks in order to cross it, but since they were unable to do so, they eventually developed a new behavior in which they learned to go around the barrier before trying to reach the prey. On a broader context, learning to detour has become a central paradigm in the understanding of animal cognition \cite{Kabadayi18}.

\subsection{Environment description}
The environment consists of a two-dimensional space constrained by a maximum value of x and y coordinates, resulting in a square-shaped boundary. In addition, there is a barrier that will prevent the agent from directly reaching the target, forcing it to detour around the barrier in order to reach the target. Thus, the set of points which the agent can move within is defined by: 
\begin{equation}
  S=\{(x,y) \in \mathbb{R}^2 \; | \; x \in [0,10] \;\& \; y \in [0,15]\}.  
\end{equation}

Finally, there is a set of environment variables, from which the set of the agent´s observation variables is derived. Also, the agent will be able to modify the value of some of these variables through its actions. These variables are:

\begin{itemize}
    \item $Target \; position$: $(x,y)$ coordinates of the Target. They are set to be $(x_T,y_T)=(10, 7.5)$ always.
    \vspace{2.5mm}
    \item $Agent \; position $: $(x,y)$ coordinates of the center of the Agent, initially they are $(x_A,y_A)=(1,7.5)$.
    \vspace{2.5mm}
    \item $Agent \; orientation$: Orientation where the Agent is looking at. In this paper, it will be $0$ always.
    \vspace{2.5mm}
    \item $Agent \; Width$ and $Agent \; Shape$: These variables indicate the shape and width of the agent. A square of side $0.75$ has been chosen for the agent.
    \vspace{2.5mm}
    \item $Barrier \; exist$: Boolean variable indicating whether or not the environment has a barrier.
    \vspace{2.5mm}
    \item $Start\; Barrier$ and $End \; Barrier$ coordinates: Coordinates of the beginning and end of the barrier. The barrier will consist of a segment between these points, made of spikes. The coordinates are $Start_B= (4.5,1.5)$ and $End_B=(4.5,15)$.
    \vspace{2.5mm}
    \item $Spikes \; separation$ and $spikes \; length$: Separation of the barrier's spikes and their length. We set $Spikes\;separation=0.5$ to prevent the agent from passing through the barrier due to its width, and $spikes \;length =0.5$ defines the length of each spike.
\end{itemize}

\begin{figure}[ht!]
        \centering
        \includegraphics[width=0.5\textwidth]{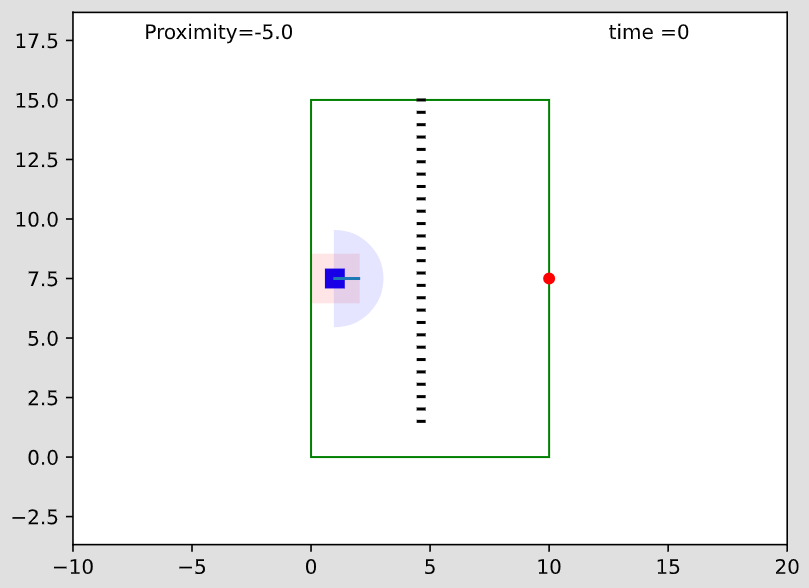}
        \caption{Initial conditions. The agent is displayed by a blue square and the target by a red dot. The pailing fence is placed in between. }
        \label{fig:Initial_conditions}    
\end{figure}
\begin{table}[ht!]
    \centering
    \captionof{table}{This tables shows the initial values of the environmnet variables}
    \begin{tabular}{cc}
            \hline
            \textbf{Environment variables} & \textbf{Initial Values} \\ \hline
            $Target \; position$                & $(10, 7.5)$                  \\
            $Agent \; position$                 & $(1,7.5)$                    \\ 
            $Agent \; orientation$              & $0$                          \\
            $Agent \; width$                    & $0.75$                       \\ 
            $Agent \; shape$                    & square                   \\ 
            $Barrier \; Exist$                  & Yes
            \\ 
            $Start \; Barrier$                  & $(4.5,1.5)$                  \\
            $End \; Barrier$                    & $(4.5,15)$                   \\ 
            $Spikes \; separation$              & $0.5$                        \\ 
            $Spikes \; length$                  & $0.5$                        \\ \hline
    \end{tabular}
\end{table}

More secondary details about how the interactions between the agent and the environment are implemented, are included in the Technical Appendix \ref{appendix:Technical Appendix}.

\subsection{Actions and Observations Space}

The agent´s actions space is the following:
\paragraph{$Step \; Forward$:} This action will allow the agent to move forward in the orientation it is facing. The decision domain of this action variable is $[0,2.5]$ and the mechanistic function behind it is the following:
\begin{equation}
    \begin{split}
        Step \; forward&\colon \mathbb{R}^2 \times [-\pi,\pi] \times[0,2.5] \longrightarrow \mathbb{R}^2\\
        \;\;\;\;SF(x,y,\alpha,s)&=(x,y)+s\cdot(\cos{\alpha},\sin{\alpha}),
    \end{split}
\end{equation}

\noindent where $(x,y)$ are the coordinates $Agent \; position$, $\alpha$ is the value of $Agent \; orientation$ and $s$ the number of steps taken by the agent.

\paragraph{$Step \; Aside$:} This action allows the agent to move perpendicular to the direction in which it is facing. The decision domain is $[-2.5,2.5]$, negative values are for moving to the left and positive values for moving to the right. The mechanistic function behind it is the following:
\begin{equation}
    \begin{split}
        Step \; Aside&\colon \mathbb{R}^2 \times [-\pi,\pi] \times[-2.5,2.5] \longrightarrow \mathbb{R}^2\\
        SA(x,y,\alpha,s)=&\left\{ \begin{array}{ll} (x,y)+s(\cos{(\alpha}-\frac{\pi}{2}),\sin{(\alpha-\frac{\pi}{2})}) & \text{si} \; s < 0 \\ \\ (x,y)+s(\cos{(\alpha+\frac{\pi}{2})},\sin{(\alpha+\frac{\pi}{2}}))& \text{si} \; s\geq 0   \end{array} \right.
    \end{split}
\end{equation}
It's important to note that while the agent doesn't know these functions, it will generate a probabilistic distribution over the observation space when a specific decision is made.\\

The agent´s observation space consists of the following variables:
\begin{itemize}
    \item $Depth$: It corresponds to the Euclidean distance between the agent and the target.
    \begin{equation}
        D=\sqrt{(x_T-x_A)^2 + (y_T-y_A)^2}
    \end{equation}
    \item $Heading \; angle$: This variable measures the angle between the agent and the target, using the agent coordinates as the reference system.
    \begin{equation}
        HA = \arctan\left(\frac{y_T-y_A}{x_T-x_A}\right)
    \end{equation}
    \item $Barrier \; Tactile$: This is a binary variable that is $1$ if the infinity norm between the center of the agent and any of the spikes of the barrier is lower than $2$, and $0$ otherwise.
    \begin{equation}
        BT= \left\{ \begin{array}{lc} 1 & \text{if} \; \underset{s \in spikes}{\max} ||(x_A,y_A)-(x_s,y_s)||_{\infty} \leq 2  \\ \\0 & \text{otherwise }   \end{array} \right.
    \end{equation}
    \item $Target \; in \; visual \; Field$: A binary variable, it will be $1$ if $Depth$ is lower than $2$ and $Heading \; angle$ is between $[-\frac{\pi}{2},\frac{\pi}{2}]$.
    \begin{equation}
        TVF(D,HA)= \left\{ \begin{array}{lc} 1 & \text{if} \; D\leq 2 \text{ and } HA \in [-\frac{\pi}{2},\frac{\pi}{2}]  \\ \\0 & \text{otherwise }   \end{array} \right.    
    \end{equation}
\end{itemize}

Since we will use a Dynamic Decision Network (DDN) to represent the agent´s knowledge, the observation and action variables must be discrete, so the agent will convert the continuous values of the observation variables to discrete values. It also needs to transform the discrete decision values taken from the action variables and select a continuous decision value. For more information about the discretization and transformation from discrete to continuous actions, we refer the reader to the Technical Appendix \ref{appendix:Technical Appendix}.

\subsection{Agent´s initial Dynamic Decision Network}

For the sake of simplicity, in this paper, we assume an initial DDN (Figure \ref{fig:Initial_DDN}) that allows the agent to perform simple target approach and simple obstacle avoidance behaviors.
The structure displayed in Figure \ref{fig:Initial_DDN} was chosen to avoid making any strong assumptions about the Learning to Detour problem. 
We refer to \cite{Paper12024a} for details on how this DDN can also be learned from scratch by interacting with the environment with the algorithms that will be introduced in section \ref{sect:Causal Discovery}. 
However, for sake of simplicity, this paper focuses on the aspects directly related to the learning  of the hidden variable. The initial DDN´s transition probabilities are detailed in the Technical Appendix \ref{appendix:Technical Appendix}.

\begin{figure}[ht]
\centering
\includegraphics[width=\textwidth]{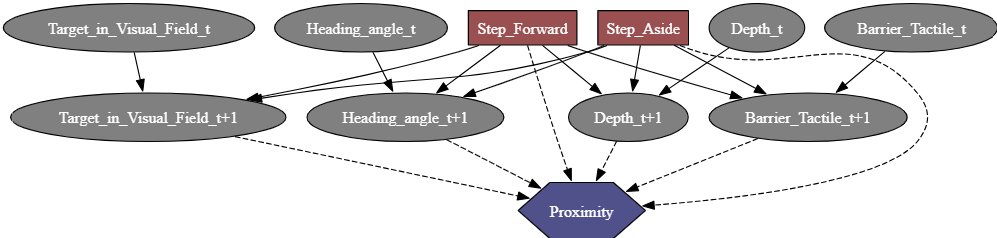}
\caption{Initial DDN structure.}
\label{fig:Initial_DDN}
\end{figure}

Regarding the DDN, the observation variables at times $t$ and $t+1$ are considered as chance nodes, the decision nodes are the action variables, and the utility node evaluates whether the agent is near the target and penalises when it hits the barrier. The utility node is formulated using the following function:
\begin{equation}\label{eq:utility_function}
   U(\textbf{x})=\left\{ \begin{array}{lc}  -2D -|HA-5| +10\cdot TVF -E_{SF} -E_{SA} & \text{if} \; BT=0 \\ \\ -10 -E_{SF} -E_{SA} & \text{if} \; BT=1   \end{array} \right.,
\end{equation}
where $\textbf{x}=(D,HA,BT,TVF,SF,SA)$ and $E_{SF}$, $E_{SA}$ are functions that return the energy wasted by the actions, depending on the value chosen by the agent, in this work we will use:
\begin{equation}
    \label{eq:energy_sf}
    E_{SF}(sf)=1+ 0.1\sqrt{sf}
\end{equation}
\begin{equation}
  \label{eq:energy_sa}
  E_{SA}(sa)=1+ 0.1\sqrt{|sa-5|},  
\end{equation}
where $sf$ and $sa$ refer to the decision values that the actions $Step\;Forward$ and $Step\;Aside$ take, respectively.\\

\begin{figure}[]
     \centering
     \begin{subfigure}[b]{0.49\textwidth}
         \centering
         \includegraphics[width=\textwidth]{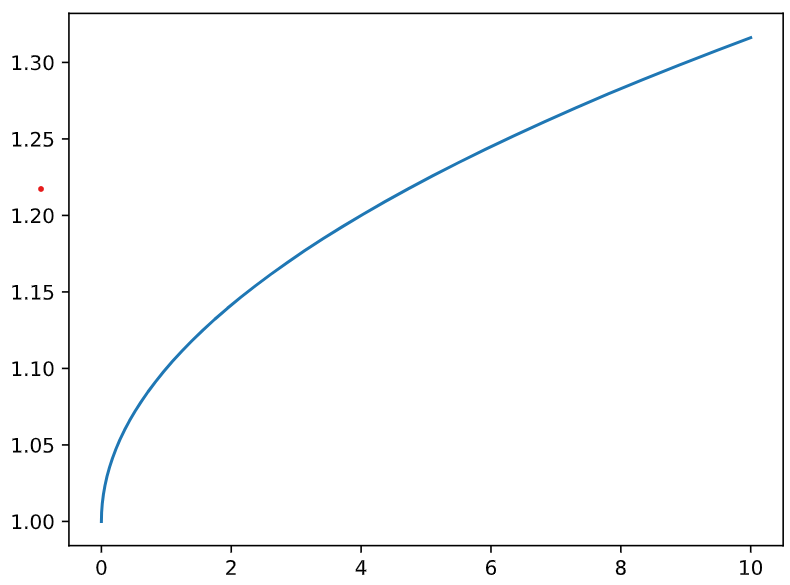}
         \caption{$Step\_Forward$ energy function (\ref{eq:energy_sf})}
         \label{fig:energy_sf}
     \end{subfigure}
     \hfill
     \begin{subfigure}[b]{0.49\textwidth}
         \centering
         \includegraphics[width=\textwidth]{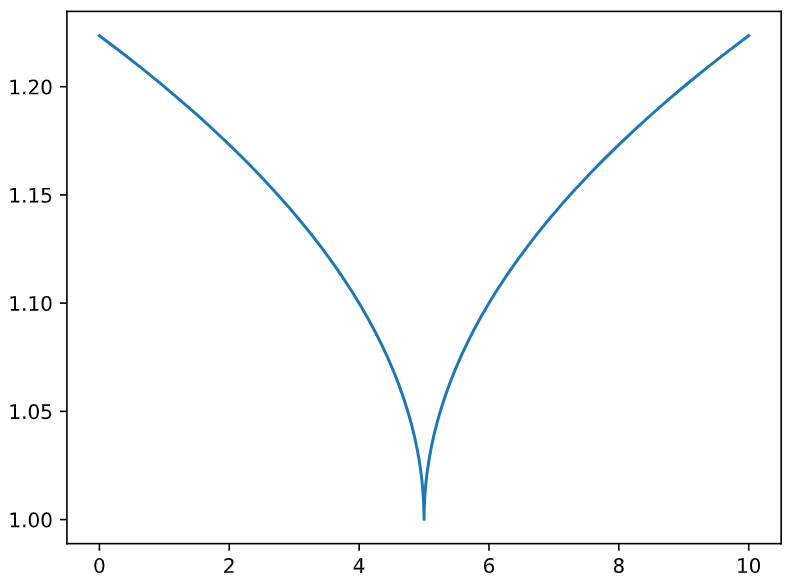}
         \caption{$Step\_Aside$ energy function (\ref{eq:energy_sa})}
         \label{fig:energy_sa}
     \end{subfigure}
        \caption{Energy functions chosen for the actions}
        \label{fig:energy_Functions}
\end{figure}
Regarding the relations between temporal variables, we assume that relations only exist between an observation variable at time $t$ and in the following time step. There are no relations between chance nodes that do not represent the same observation variable. Furthermore, decision nodes are related to all $t+1$ chance nodes.

\subsection{Problem modeled as a POMDP}
If we would model the interactions of the agent with the environment as a POMDP, we would have the following elements:
\begin{itemize}
    \item $S=\{(x,y) \in \mathbb{R}^2 \; | \; x \in [0,10] \;\& \; y \in [0,15]\}$. 
    \vspace{2mm}
    \item $A=\{Step \;Forward \; ,Step\;Aside\}$.
    \vspace{2mm}
    \item $T$ would be the degenerate distribution, which gives probability $1$ to the new state $s^\prime$ determined by the mechanistic function of the executed action $a$, according to the state $s$ of the agent.
    \vspace{2mm}
    \item $R$ would be the same as described in formula (\ref{eq:utility_function}).
    \vspace{2mm}
    \item $\Omega =\{Depth \; ,Heading\;Angle\;,Barrier\;Tactile,\;Target$ $\;in\;Visual\;Field\}$.
    \vspace{2mm}
    \item $O$ would also be the degenerate distribution, assigning probability $1$ to the value given by the function of the observation variable. In other words, the agent has accurate sensors that do not introduce errors into the measurement process. 
    \vspace{2mm}
    \item Since we are going to use a DDN that contemplates only the next instant in the future, the $\gamma$ factor is equal to $0$.
\end{itemize}


\section{Dealing with new Causal relations and Unexpected patterns}

This section briefly describes, for completeness reasons, the main aspects of Causal Discovery, as well as the basic ingredients of our Theory of Surprise. Yet, we refer the reader to \cite{Paper12024a,Paper22024b} for more detailed explanations. 

\subsection{Causal Discovery}
\label{sect:Causal Discovery}

Causal discovery in models of animal behavior consists, among other things, of learning the relation between the observed effects and the hypothetical causes that gave rise to them, for instance \cite{Corbacho1997,Corbacho2005}. The field of causal discovery in machine learning and statistics is a very active open research problem \cite{Glymour2019,Guyon2019,Lopez-Paz2015,Pearl2000,Peters_2011,Peters2015,Peters2017,Squires2023}. 
In this paper, for simplicity reasons, we assume that the agent is already equipped with several causal relations. Nevertheless, we will introduce the basics and ideas behind the algorithms that we would use in the case of complete learning.
We refer the reader to \cite{Paper12024a} for details on the algorithms that  learns this DDN from scratch by interacting with the environment.\\

In order to quantify and detect relationships between random variables, it is necessary to give a value to the relationship, that tells us how strong or true it is. To do this, we will use the causal coefficients.

\begin{definition}\label{def:Definición causalidad}Causal Coefficients\\
Given two random variables $X$ and $Y$, a causal coefficient $C(X,Y)$ is a scalar value such that, the larger $C(X,Y)$ is, the more truthful is the relation  $X\rightarrow Y$.\\

$C(X,Y)$ must satisfy the following properties:
\begin{enumerate}
    \item Antisymmetry: \[C(X,Y)=-C(Y,X)\]
    \item Discriminant: 
    \[C(X,Y)> \alpha > 0 \Rightarrow X\xrightarrow{} Y.\]
    The parameter $\alpha$ has the role of determining when the coefficient is large enough to consider the relationship $X\rightarrow Y$ to exist. If $|C(X,Y)|<\alpha$, then the relationship between $X$ and $Y$ cannot be determined.\\
    \item Invariant to linear transformations: 
    \[C(aX +b, cY + b)=C(X,Y) \;\; \forall a,b,c,d \in \mathbb{R}, \;\;  a,c\neq 0\]
\end{enumerate}
\end{definition}

From this definition, different coefficients can be designed, but the one we decided to use is the following:
\begin{equation}
\label{eq:Coef_entropy}
   C(X,Y)=\frac{H(X)}{log|X|} - \frac{H(Y)}{log|Y|},
\end{equation}
where $H(X)$ is the entropy of the random variable $X$ (\ref{eq:entropy}).
On one hand, the coefficient (\ref{eq:Coef_entropy}), arises from the work of Janzing et al. \cite{IGCI}, where it is shown that if $X \xrightarrow{} Y$ and there is independence between the mechanistic function and the distribution of $X$, then under some conditions one has that $H(X)\geq H(Y)$.
On the other hand, this coefficient takes into account that random variables with higher cardinal numbers may have higher entropy values than those with lower cardinal numbers. By normalizing by the logarithm of the cardinal, which is the maximum value of the entropy, we avoid cases in which the causality coefficient is high only because of the difference in cardinality between the variables and not because there is a causal relationship.\\

In the same manner as with causality between random variables, it  is desirable to have a causality coefficient that allows us to quantify the strength of the causal relationship between an action variable and a random variable. The definition \ref{def:Definición causalidad} will be used as a baseline for this purpose. 
\begin{definition}Causal action coefficient\\
Given an action variable $D$ and an observation variable $O$, a causal action coefficient $C_D(O)$ is a scalar value such that, the larger $C_D(O)$, the more truthful is the relation $D\rightarrow O$.\\

\noindent $C_D(O)$ must satisfy the following properties:
\begin{enumerate}
    \item Non-negativity: \[C_D(X,Y)\geq 0\]
    \item Discriminant: 
    \[C_D(O)> \alpha > 0 \Rightarrow D\xrightarrow{} O,\]
    where $\alpha$ will fulfill the same role it has in the definition of causal coefficient (\ref{def:Definición causalidad})\\
    \item Invariant to linear transformations: 
    \[C_D(aO +b)=C_D(O) \; \forall a,b\in \mathbb{R} , a\neq 0\]
\end{enumerate}
\end{definition}

A possible example of a causal action coefficient would be:
\begin{equation}
    C_D(O)=H(O)-H(O|D),
\end{equation}
where the entropy of $O$ (\ref{eq:entropy}) and the conditional entropy of $O$ given $D$ are used.\\

As for the algorithms, there will be two of them, one that builds the graph corresponding to the intratemporal relationships $G_0$, and another one that builds the intertemporal relationships involving each of the actions, i.e., it will build a graph for each action $G_{\xrightarrow{},Act}$.
The algorithm that will build $G_0$ will use a causal coefficient as defined above \ref{eq:Coef_entropy}. Yet, the algorithm that deals with the intertemporal relationships for each action, will need to calculate the temporal causal relationships between the observation variables and the causality between the action variables at time $t$ and the observation variables at time $t+1$. In this regard, the transfer entropy will be used.

\begin{definition}\label{def:Transfer_entropy} Transfer entropy between two stochastic processes\\
Given two stochastic processes $X_t$ and $Y_t \;  \forall t\in \mathbb{N}$, the transfer entropy is defined as:
\begin{equation}
\label{eq:Coef_transfer_entropy}
    T_{X \xrightarrow{} Y}=H(Y_{t}|Y_{t-1},\dots,Y_{t-l}) - H(Y_{t}|Y_{t-1},\dots,Y_{t-l},X_{t-1}\dots,X_{t-l}), \;\; \forall l\geq 1  
\end{equation}
This coefficient measures the direct information transfer between two stochastic processes $X$ and $Y$ and has a strong association with the Granger causality between $X$ and $Y$ \cite{Amblard_2013,Barnett_2009,Schlinder_2011}. 
\end{definition}

The $l$ chosen will depend on the time in the future that we intend to detect causality between actions at time $t$ and observation variables at time $t+l$, as between observation variables between $t$ and $t+l$.
The causality coefficient that we have decided to use is as follows:
\begin{equation}
\label{eq:Transfer_Entr_norm}
    NormT_{(Act,\textbf{O}) \xrightarrow{} Obs}=\frac{H(Obs_{t+1}|Obs_t) - H(Obs_{t+1}|Obs_{t},Act_t,\textbf{O})}{H(Obs_{t+1}|Obs_t)}
\end{equation}
This normalization factor results from working with percentages of entropy reduction relative to the entropy of the $Obs$ variable alone. The greater the reduction in entropy, the more information the variables contribute to the explanation of the $Obs$ variable. 
In this regard, the algorithm will work as a forward algorithm, that is the observation variables are added, apart from the action variable, according to the value of the partial entropy reduction; 
if the partial reduction is not enough, the process of adding variables is stopped.

\subsection{Surprise Divergence}
\subsubsection{Motivation for the surprise divergence measure}
Our theoretical framework for detecting latent variables will be based on the concept of the agent's surprise. 
In order to do this, we must first define what surprise is. 
We define it as the response to an unexpected event, so surprise comes from comparing what was believed or predicted to happen with what actually happens.\\

In mathematical terms, we could say that surprise comes from comparing the probability distribution generated by an agent of what will happen in the future, with the value that the variable ends up taking. Yet,  this can be generalized to consider the comparison of two distributions, the one that is thought to occur and the ``probability distribution'' that actually ends up occurring.
Therefore, it is clear that we need a divergence that compares the two distributions. There are several divergences available, such as the Kullback-Leibler divergence, but we choose to define our own from information theory concepts due to reasons later explained. 
We will use both entropy (expectation of the variable information) and information dispersion (standard deviation of the variable information) to standardize the information values of the events according to a distribution. In this way, we will quantify how far (in number of deviations) the events are from what we would consider predictable, i.e., we will weight how surprising these events would be. Then we will weight these values with the probabilities of the second distribution to which we want to compare. If high standardized values from one distribution have a high probability in the other distribution, then the distributions will be very different.
We notice that our approach will generate an asymmetric divergence, since it is very important to know which of the two distributions is the reference, since the surprise values will depend on it. 

\subsubsection{Mathematical formalism of surprise divergence measure}
Before introducing the surprise divergence measure, it is important to mention some concepts from Shannon´s information theory.
Let $X$ be a discrete random variable with $k$ possible outcomes $x_i \; i=1,\dots,k$, and a probability distribution over these outcomes $P(X)=(p_1,\dots,p_k)$, where $p_i=P(X=x_i)\; i=1,\dots,k$.
The entropy of $X$ based on the distribution $P(X)$ is 
\begin{equation} \label{eq:entropy}
    H(X)=E[-log(P(X))]=-\sum_{i=1}^k p_i log(p_i) 
\end{equation}

Understanding the concept of entropy as the first moment of the random variable $X$, we can extend it to higher order moments. Specifically, we define the information dispersion as the central moment of order 2 based on $P(X)$:

\begin{equation}\label{eq:information dispersion}
    \begin{split}
        V_I (X)=E[\left(-log(P(X))-H(X)\right)^2]=\\=\sum_{i}^k p_i (log(p_i))^2 - H(X)^2
    \end{split}
\end{equation}
Sometimes we will denote entropy and information dispersion by $H(P)$ and $V_I(P)$ respectly to emphasize that the distribution $P$ was used to calculate the moments.
Lastly, given a discrete random variable $X$ and two different distributions $P$, $Q$, the Kullback-Leibler divergence between $P$ and $Q$ is a measure of how different is a distribution in relation with the other one:
\begin{equation}
\label{eq:div_KL}
    D_{KL}(P||Q)=-\sum_{i=1}^k p_i log\left(\frac{q_i}{p_i}\right) =H(P,Q)-H(P),
\end{equation}
where $H(P,Q)$ denotes the cross entropy between $P$ and $Q$.\\

After introducing these ideas, we define the surprise divergence between two probability distributions $P$, $Q$ for the same random variable $X$ as:
\begin{equation}\label{eq:Surprise Divergence}
    \begin{split}
        D_S(Q||P)=E_Q\left[\frac{-log(P(X)) - H(X_P)}{\sqrt{V_I(X_P)}}\right]=\\= \frac{-\sum_{i=1}^{k} q_i\cdot log(p_i) -H(X_P)}{\sqrt{V_I(X_P)}}=\frac{H(Q,P) -H(P)}{\sqrt{V_I(P)}} \text{.}
    \end{split}
\end{equation}

Upon reformulation, one can express $D_s(Q||P)$ as the sum of the Kullback-Leibler  (\ref{eq:div_KL}) divergence between $Q$ and $P$, and the difference in entropies of $Q$ and $P$, all divided by the square root of the information dispersion of $P$:
\begin{equation}
\label{eq:reformulation}
\begin{split}
    D_S(Q||P)&=\frac{H(Q,P) -H(P)}{\sqrt{V_I(P)}} = \frac{-\sum_{i=1}^k q_i log(p_i) -H(P)}{\sqrt{V_I(P)}}=\\
    &=\frac{-\sum_{i=1}^k q_i \left(log\left(\frac{p_i}{q_i}\right) + log(q_i)\right) -H(P)}{\sqrt{V_I(P)}}=\\
    &=\frac{D_{KL}(Q||P) +H(Q)-H(P)}{\sqrt{V_I(P)}}.
\end{split}
\end{equation}

This reformulation (\ref{eq:reformulation}) enables us to observe that the surprise divergence does not solely depend on the Kullback–Leibler divergence to evaluate the discrepancy of $Q$ in relation to $P$:
\begin{enumerate}
    \item $H(Q)-H(P)$ measures the difference in the uncertainty given by both distributions.
    \item $\frac{1}{\sqrt{V_I(P)}}$ serves as a catalyst, the smaller the value of information dispersion of $P$, the more relevant will be the differences between the distributions, however small they may be.
\end{enumerate}

Finally, we will present and prove a number of propositions and theorems that are essential for understanding the rationale behind the decisions we will be making in subsequent sections.

\begin{prop}\label{prop:D_s_tend_0}
$D_s(Q||P)=0 \Leftrightarrow Q=P$.
\end{prop}
\begin{proof}
    See proof in Appendix \ref{Appendix divergence proofs}, proposition \ref{proof prop1}
    
\end{proof}

\begin{thm}
\label{theorem:Expect_max_Divergencia_S}
    $D_{KL}(Q||P) \xrightarrow{} 0$ $\Leftrightarrow$ $D_S(Q||P)^2 \xrightarrow{} 0$.
\end{thm}
\begin{proof}
    See proof in Appendix \ref{Appendix divergence proofs}, theorem \ref{proof theorem1}
\end{proof}

\begin{corollary}
It can be concluded from this theorem that the minimisation of $D_{KL}$ is equivalent to the minimisation of $D_S^2$, and therefore also equivalent to the maximisation of the likelihood. This guarantees that the Expectation-Maximisation algorithm can be used to minimise $D_S(Q||P)^2$.
\end{corollary}

\begin{thm}
    \label{theorem:Div_normal}
     Let $\{X_1,\dots, X_n\}$ be a sample of independent random variables, each one with the same probability distribution $P=[p_1,\dots,p_k]$, then
     \[\sqrt{n} \cdot D_S(\hat{P}||P) \xrightarrow{d} N(0,1),\] where $\hat{P}=[\hat{p}_1,\dots,\hat{p}_n]$ is de maximum likelihood estimator of $P$ from $\{X_1,\dots, X_n\}$.  
\end{thm}

\begin{proof}
See proof in Appendix \ref{Appendix divergence proofs}, theorem \ref{proof theorem2}

\end{proof}

\begin{prop}\label{prop:prob_tend_info_disp}
    Let $\{X_1,\dots, X_n\}$ be a sample of independent random variables, each one with the same probability distribution $P=[p_1,\dots,p_k]$ and 
    \[\hat{V}_I(P)=\frac{1}{n-1}\sum_{i=1}^n (-log(P(X=X_i))-H(\hat{P},P)^2),\] then:
    \[\hat{V}_I(P) \xrightarrow{P} V_I(P).\]
\end{prop}

\begin{proof}
    See proof in Appendix \ref{Appendix divergence proofs}, theorem \ref{proof prop2}
\end{proof}

\begin{thm}
    Given a sample $\textbf{X}$ as in theorem \ref{theorem:Div_normal}, then
    \[\sqrt{n} \cdot \frac{H(\hat{P},P) - H(P)}{\sqrt{\hat{V}_I(P)}} \xrightarrow{d} N(0,1).\]
\end{thm}
\begin{proof}
    See proof in Appendix \ref{Appendix divergence proofs}, theorem \ref{proof theorem3}
\end{proof}

\subsection{Surprise coefficient}
From the surprise divergence between two distributions \ref{eq:Surprise Divergence},  we define the surprise of an outcome of $X$ relative to the distribution $P(X)$, as the following coefficient:
\begin{equation}\label{eq:surprise_coefficient}
    C_S( x_i|| P(X))=\frac{|-log(p_i)-H(X)|}{\sqrt{V_I(X)}}
\end{equation}

This is a particularization of \ref{eq:Surprise Divergence}, where $Q$ would be the degenerate distribution of $X$ with $qi = 1$ and $q_j=0 \; \forall j \neq i$. Additionally, the absolute value of the coefficient is also taken. A paper in preparation \cite{Paper22024b} will provide a detailed comparison with other theories of surprise \cite{Baldi2002,BALDI2010,Friston2010,Friston2017,Itti2005,Ku2015,MODIRSHANECHI2022,SCHMIDHUBER2009,Schwartenbeck2013}. 

\section{Learning to Detour Overall Process}
In our general learning theory  we try to assume the least possible amount of a priori knowledge and let the agent, by interacting with its environment, learn all the necessary causal structure as it is needed on the fly. The overall learning to detour process goes through the following main phases:

\begin{enumerate}
    \item \textbf{First phase: Causal Discovery}
        \begin{enumerate}
            \item Random action execution interacting with the environment to generate samples
            \item Learning of Dynamic Decision Network (DDN)
                \begin{enumerate}
                    \item Execute the algorithm for intratemporal relationships outlined in section (4.1), and detailed in  \cite{Paper12024a}.
                    \item For each action, execute the algorithm for intertemporal relationships outlined in section (4.1), and detailed in \cite{Paper12024a}.
                \end{enumerate}
            \item Learn the parameters of DDN CPTs, using MLE or MAP methods.
        \end{enumerate}
    \vspace{1mm}
    \item \textbf{Second phase: Learning Latent Variables:} 
        \begin{enumerate}
            \item MEU interaction with the environment to generate samples. In each iteration:
            \begin{enumerate}
                \item Collect evidence over observation variables
                \item Action selection by Maximum Utility Expectation (equation \ref{eq:MEU}) using an optimizing algorithm (for example Shafer Shenoy for LIMIDs)

                \item Compute the Expected Utility and the probability distribution associated with the Utility values $U(x)$ (equation \ref{eq:utility_function}) before executing the selected action. 
                \item Execute the action, collect the new evidence in the observation variables, and calculate the Utility value
                \item Compute the surprise in the Observation variables and in the Utility with coefficients 
                (\ref{eq:surprise_coefficient}),
                (\ref{eq:Utility_surprise}), respectively                \item Calculate the probability of Influence of the latent variable (equation \ref{eq:probability_influence_HV}) and check if the agent detected a latent variable with the test of hypothesis (test \ref{eq:hypothesis test}). 
                \item In case of rejection, store the variables related to the new latent variable, to later add the new hidden variable and its (causal) relations.
            \end{enumerate}
            \item At the end of the epoch of interaction with the environment:
            \begin{enumerate}
                \item Add the hidden variables, with their respective relationships to the observational variables, to the agent´s DDN.
                \item Execute the algorithm weighted EM (algorithm \ref{alg:weighted EM}) for learning the new CPTs associated with the new hidden variables until convergence.
            \end{enumerate}
        \item Repeat (a) and (b), skipping step (b.i), until convergence of new CPTs added.
            
        \end{enumerate}
\end{enumerate}	

\subsection{Latent variable detection and new Hidden variable Structure learning}
As we have mentioned, we propose latent variable detection when a significant surprise in the utility function takes place.  It follows with hidden variable sub-graph structure learning to adapt the structure of the DDN graph to the new detected hidden variable and its topological structure with respect to the already existing chance variables in the DDN.  The selection of candidates for the hidden variable´s children and parents is also based on the relevance of each particular chance variable on the coefficient of surprise for the overall probabilistic graph joint distribution. For this paper we assume a “XM” topological structure that we will define later. 
The possible existence of a latent variable is detected by a sudden significant surprise in the utility function (the symptom). Yet the cause has to be inferred. In the Learning to Detour case, the latent cause corresponds to the width of the agent, which happens to be wider than the paling fence openings, thus making all the gaps non passable. 

\subsubsection{Utility Surprise}
This section describes latent variable detection by Surprise in the Utility function. That is, after observing the actual Utility, the agent will compute the possible surprise in the Utility. In order to do so, we extrapolate the surprise coefficient defined in (\ref{eq:surprise_coefficient}) to define the Utility surprise coefficient, where we also take into account whether a specific surprise generates an utility larger or smaller than expected.\\

First of all, recall that the agent's way of acting is summarized by observing the values of the observation variables $\textbf{x}_{t-1}$ at time $t-1$ and computing the probabilities on the observation variables at time $t$ from these values, as well as the possible values of the action space $Act_{t-1}$, denoted by $\probP(\textbf{x}_t |\textbf{x}_{t-1}, Act_{t-1})$.  
This enables the agent to calculate all possible values of $U(\textbf{x}_t)$ and then choose the decision $\textbf{a}^*$ that brings the maximum expected utility,
\begin{equation}
    \label{eq:MEU}
    MEU(t)= \max_a E_{P(\textbf{x}_t | \textbf{x}_{t-1} , a)}[U(x_t)]= E_{P(\textbf{x}_t | \textbf{x}_{t-1} , \textbf{a}^*)}[U(x_t)].
\end{equation}

The way the agent will detect the latent variables will be by comparing the utility value it receives at time $t$ with the probability distribution $P(\textbf{x}_t |\textbf{x}_{t-1}, Act_{t-1})$ it computed at time $t-1$. Two critical aspects must be taken into account in this comparison:
\begin{enumerate}
    \item On the one hand, it is important to notice whether the obtained value $U(x_t)$ is greater or smaller than the expected, $MEU(t)$.
    \item On the other hand, we are interested in how surprising the value is compared to the distribution $\probP(\textbf{x}_t | \textbf{x}_{t-1}, Act_{t-1})$.
\end{enumerate}

As a result, there are four possible scenarios in which to consider the probability that a latent variable was involved. 
\begin{enumerate}
    \item large surprise and negative utility
    \item large surprise and positive utility
    \item low surprise and negative utility
    \item low surprise and positive utility
\end{enumerate}

Case $1$ means that the utility value obtained had a very low associated probability, and if, in addition, the utility value is lower than the mean, then it means that there was some unobserved variable that influenced the agent's action, causing the utility not to be maximized as expected, and making it convenient for the agent to learn to predict when something similar might occur so that the decision is correct. The reason for studying whether the expected value is lower or higher is that, in the end, the agent's choice of decisions is conditioned to look for the highest expected value. So,  if it turns out that the value obtained is lower, then the agent could have chosen another decision that might have yielded a higher utility. It is also irrelevant whether the dispersion of the distribution was high or low, since the surprise coefficient indicates whether the utility obtained was a surprise event or not.\\

As for the case $2$, if the utility obtained is higher than expected, then we would have the opposite case, i.e. it has been discovered that there is a variable that influences and helps to obtain higher utility values, so it is convenient to learn to model when this may occur. However, in this paper, we will focus on detecting the latent variables that have a negative influence on the utility, and therefore the probability of influence will be close to $0$, in the case that the utility obtained is higher than expected.\\

If the surprise coefficient is close to $0$, as for the cases $3$ and $4$, then it means that the probability that there is a variable that the agent does not observe, and that changes its interaction with the environment will be low, since the agent is able to correctly predict the values of the observation variables that it will get after executing the actions.
Taking all this into account, the coefficient that will be responsible for making this comparison will be an extrapolation of the one described in (\ref{eq:surprise_coefficient}), which will also include whether the utility obtained is higher or lower than the one expected.
\begin{equation}\label{eq:Utility_surprise}
C_U(\textbf{P}||\textbf{x}_t)=sign\left(U(\textbf{x}_t)-MEU(t)\right)\cdot C_S(\textbf{P}||\textbf{x}_t),
\end{equation}

where $U(\textbf{x}_t)$ corresponds to the value of the utility function for the observation $\textbf{x}_t$ and $\textbf{P}$ corresponds to the probability distribution $P(\textbf{x}_t |\textbf{x}_{t-1}, Act_{t-1})$.
This surprise coefficient can take  large values when there is poor estimation probabilities or when a latent variable is influencing the current situation. In this paper, for simplicity reasons we assume the second case. 

\subsubsection{Probability of Influence of the Hidden Variable}
\label{section:IPHV}
Next, the agent must compute the probability of influence of the latent variable on the Utility. That is, the probability that a change in the latent variable is the cause of the observed surprise in the utility function, as well as, influencing the observations at $t+1$.
Internally, the latent variable is represented by a hidden variable in the DDN. It is a hidden variable since it never gets any associated observation.   In this paper, for simplicity reasons, we assume that the hidden variable is binary. We use the value 0 to represent that the hidden variable does not have any influence on the DDN, and, thus, the original model correctly explains the Observation variables distributions and also allows to properly predict the expected utility. On the other hand, the value 1 represents that the hidden variable does influence the DDN, and a modification on the distributions (transitions) at $t+1$  is needed, since they were not correct, and also a negative surprise on the expected utility took place.\\

The probability of influence of the latent variable will be a function dependent on the value of the coefficient of the utility surprise. The range of possible values that $C_U$ can take is potentially $(-\infty ,\infty)$, and the goal is to convert it to the range $[0,1]$.
Therefore, assuming we want to calculate $\probP(HV=0)$, the lower the coefficient, the greater the negative surprise, that is, the greater the evidence that $HV=1$, so $\probP(HV=0)$ must tend to $0$ . On the other hand, the higher it is, the greater the evidence that $HV=0$, thus $\probP(HV=0)$ must tend to $1$.
However, when $C_U$ is close to $0$, this will indicate that there has not been enough surprise at instant t, so we do not have enough evidence to say that $HV=1$, but neither is it correct to say that $HV=0$, since it could happen that the influence of the latent variable is practically null, therefore, when $C_U=0$ we will have that $\probP(HV=0)= \frac{1}{2}$.\\

Summing up, the range of possible values that $C_U$ can take is potentially $(-\infty ,\infty)$, and the goal is to convert it to the range $[0,1]$ as long as $\probP(HV=0|C_U=0)= \frac{1}{2}$, $\probP(HV=0|C_U\xrightarrow{}\infty)=1$ and $\probP(HV=0|C_U\xrightarrow{}-\infty)=0$.
This can be achieved by different functions (such as the logistic, hyperbolic tangent, normal cumulative distribution function, etc). However, we will choose a function that is consistent with the underlying statistical theory behind the coefficient.
The absolute value $|C_U|$, based on the theorem \ref{theorem:Div_normal}, would statistically follow a $\chi$ distribution with one degree of freedom, since its the absolute value of something that tend to a normal distribution. Therefore, we will use the cumulative distribution function of a $\chi$ distribution as a baseline. In addition, we must consider the sign of the coefficient and guarantee that $C_U=0$ corresponds to a probability of $\frac{1}{2}$. To this end, we have developed the following function:
\begin{equation}\label{eq:probability_influence_HV}
    P(HV=0 |C_U) =\left\{ \begin{array}{lc} \frac{1}{2}-\frac{1}{2}P(\chi \leq |C_U|) & \text{if} \; C_U < 0 \\ \\ \frac{1}{2}+\frac{1}{2}P(\chi \leq C_U) & \text{if} \; C_U\geq 0   \end{array} \right.,
\end{equation}
which behaves as we have previously described. That is, when  $C_U << 0$, the probability of no influence will tend to $0$; meanwhile if $C_U >> 0$, the probability will tend to $1$.

\begin{figure}[ht!]
    \centering
    \includegraphics[width=0.5\textwidth]{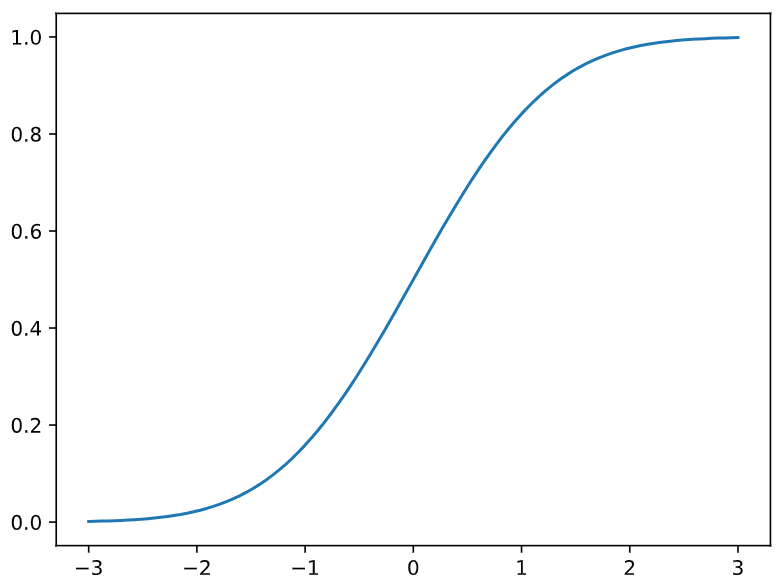}
    \caption{Plot of the function( \ref{eq:probability_influence_HV})}
    \label{fig:plot_PIHV}
\end{figure}

\subsubsection{Selection of related variables: Surprise in Observations}
Once we have determined the existence of a latent variable by the previous calculations, the next step is to introduce the hidden variable in the DDN´s graph. In order to do so, we must first select the subset of observation variables that are related to the hidden variable. To do so, we must determine the degree of relevance of each variable with respect to the overall surprise in the conjoint probability distribution.
So, we need to compute the possible surprise for each observation variable. That is, the surprise coefficient produced between the predicted distribution for an Observation variable at time $t+1$ and the real observed value at  $t+1$.
The Observation variables that have just undergone a large enough surprise will be selected as the children and parents of the hidden variable.\\

In order to identify which variables exhibit sufficient surprise during an epoch of experiences, we have developed a statistical hypothesis test based on the properties of the surprise divergence (\ref{eq:Surprise Divergence}).
Briefly explained, this test allows to determine whether a distribution fits a sample, or it does not, by studying the surprises that would be experienced if it were true. Variables that reject this hypothesis are identified as parents and children of the hidden variable. If no variable rejects this hypothesis, it can be concluded that there is no need to include a hidden variable, since there would be no influence of a latent variable to be considered.\\

Therefore, the test is defined by:
\[\;\;\;\;\;H_0: \text{The variable } Obs \text{ has not experienced a surprise in time } t \]
\[H_1: \text{The variable } Obs \text{ has experienced a surprise in time } t,\]

using the surprise divergence (\ref{eq:Surprise Divergence}), the test is as follows:

\begin{equation}
    \label{eq:hypothesis test}
    \begin{split}
        H_0: &D_S(Obs_t || P(Obs_{t}|Obs_{t-1},Act_{t-1}))=0\\
        H_1: &D_S(Obs_t || P(Obs_{t}|Obs_{t-1},Act_{t-1}))\neq 0,    
    \end{split}
\end{equation}
    
and considering theorem \ref{theorem:Div_normal}, the statistic and the critic region of the test are: 
\begin{equation}
C_S=\left|\frac{-log(Obs_{t})-H(P(Obs_{t}|Obs_{t-1},Act_{t-1}))}{V_I(P(Obs_{t}|Obs_{t-1},Act_{t-1}))}\right|    
\end{equation}
\begin{equation}
R=\left\{|C_S|\geq z_{1-\frac{\alpha}{2}}\right\},    
\end{equation}
which leads us to the formula for the p-value of the hypothesis test
\begin{equation}
    p\_value=2\cdot (1-P(Z\leq C_S)),
\end{equation}
with $Z\sim N(0,1)$.\\

Returning to the Learning to Detour case (Figure \ref{fig:Normalized surprises coefficient}), when the agent first bumps into the barrier, an unexpected effect on the $Barrier \; Tactile$ Observation variable occurs, since it is not predicted by any existing transition. Similarly, an unexpected effect on the $Depth$ Observation variable occurs after the activation of the $Step \; Forward$ Action variable, since the associated transition predicts a reduction on the $Depth$ value, yet the $Depth$ value remains unaltered  since no forward movement can take place, due to the blocking effect by the barrier.
For the purposes of visualization, the surprises were normalized by dividing them by the sum of the surprises present in the observed variables. This allowed for the accurate representation (Figure \ref{fig:Normalized surprises coefficient}) of the fact that the variables $Barrier\;Tactile$ and $Depth$ experienced the greatest degree of surprise, and therefore represent the variables that should be selected.
\begin{figure}[ht!]
\centering
\includegraphics[width=0.75\textwidth]{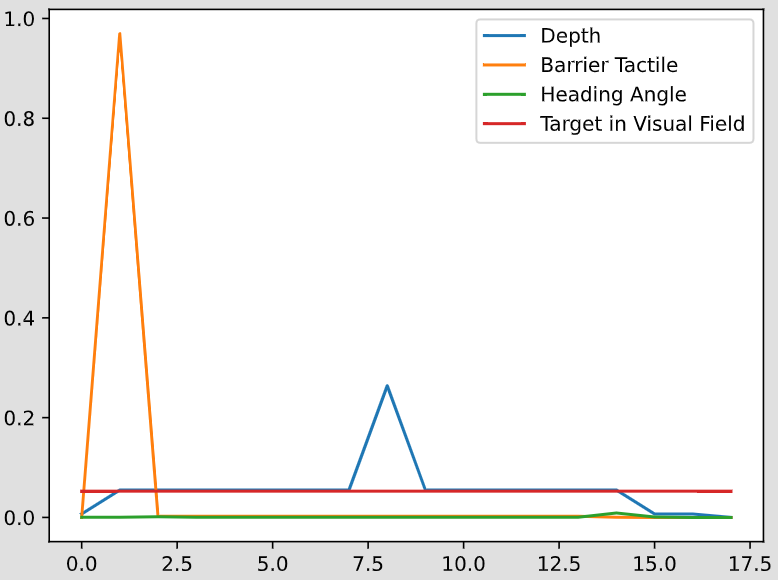}
\caption{Normalized surprise coefficients over time.}
\label{fig:Normalized surprises coefficient}
\end{figure}

\begin{figure}[ht!]
     \centering
     \begin{subfigure}[b]{0.49\textwidth}
         \centering
         \includegraphics[scale=0.3]{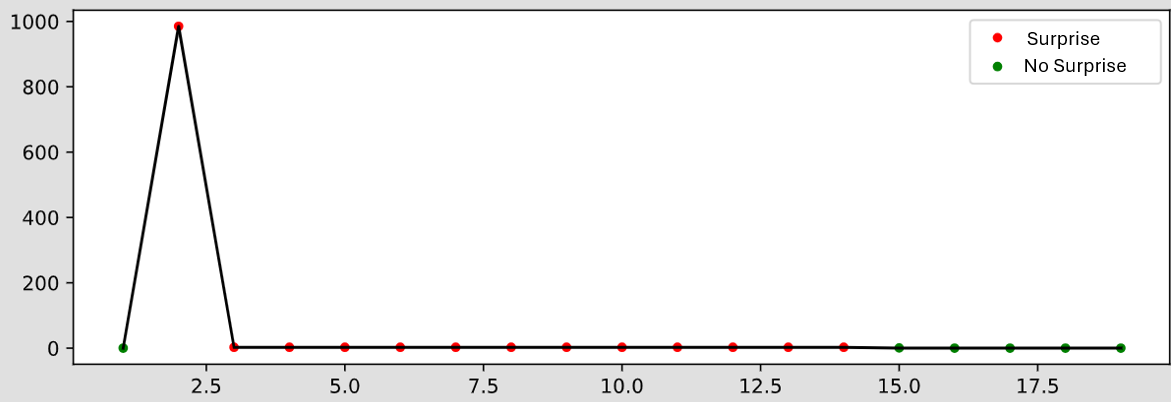}
         \caption{Surprises $Barrier\;Tactile$}
         \label{fig:Barrier_Tact_sorp_epoca_1}
     \end{subfigure}
     \hfill
     \bigskip
     \begin{subfigure}[b]{0.49\textwidth}
         \centering
         \includegraphics[scale=0.3]{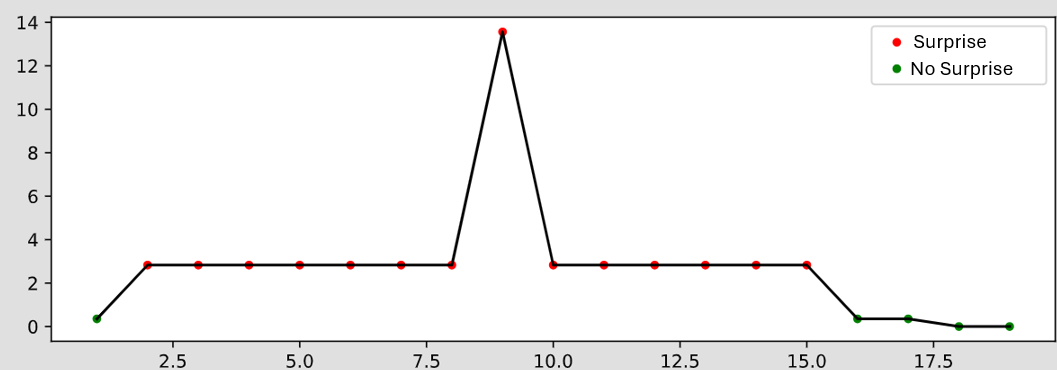}
         \caption{Surprises $Depth$}
         \label{fig:Depth_sorp_epoca_1}
     \end{subfigure}
     
     \bigskip
     \begin{subfigure}[b]{0.49\textwidth}
         \centering
         \includegraphics[scale=0.3]{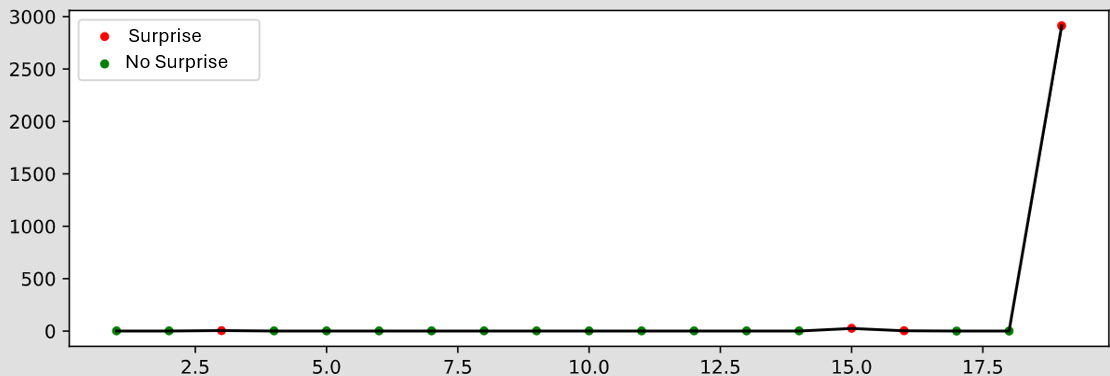}
         \caption{Surprises $Heading \; angle$}
         \label{fig:Heading_angle_sorp_epoca_1}
     \end{subfigure}
     \hfill
     \bigskip
     \begin{subfigure}[b]{0.49\textwidth}
         \centering
         \includegraphics[scale=0.3]{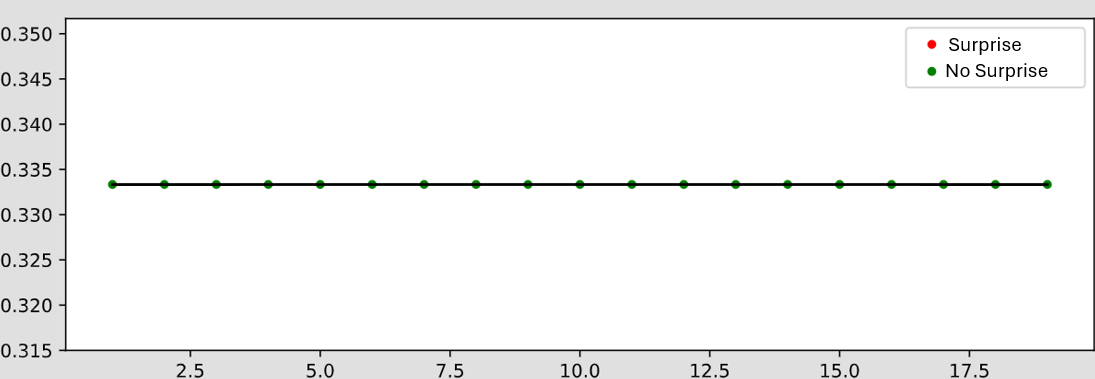}
         \caption{Surprises $ Target \;in\;Visual\;Field$}
         \label{fig:Target_in_Visual_Field_sorp_epoca_1}
     \end{subfigure}
     
    \caption{In figures(\ref{fig:Barrier_Tact_sorp_epoca_1},\ref{fig:Depth_sorp_epoca_1},\ref{fig:Heading_angle_sorp_epoca_1},\ref{fig:Target_in_Visual_Field_sorp_epoca_1}) are represented the values of $C_S$ for each observation variable throughout the iterations. Green dots mark the iterations in which the hypothesis test on the observation variables is not rejected and in red the iterations in which it is rejected. The graphs show only the first $18$ iterations of the epoch.}
    \label{fig:Sorpresas_epoca1}
\end{figure}

\subsubsection{Structure Learning with Hidden Variable}

Next, the hidden variable sub-graph (children and parents edges) introduction in the DDN takes place by structure learning. That is, a subset of the chance variables were selected as children and parents of the hidden variable in the previous sections by significant surprise detection. 
Now, we must determine the specific sub-graph topology. In this paper, for simplicity reasons we assume a “XM” topology (Figure \ref{fig:DN_with_HV}), based on the following facts:
\begin{enumerate}
    \item As we previously mentioned, the hidden variable at time $t$ can not be observed directly, yet we can assign a probability distribution based on the values of the $Obs_t$ related with the latent variable. This means that, based on the mechanistic causality definition, there exists a function such that the $HV$ values can be generated from the values of $Obs_t$ and a random noise, so there must be edges from $Obs_t$ to $HV$.
    \vspace{2mm}

    \item On the other hand, as discussed in section \ref{section:IPHV}, the latent variable influences the probability distribution of the observed variables at $t+1$, causing a surprise observation at $Obs_{t+1}$. Therefore, the latent variable influences the generation of the $Obs_{t+1}$ values, and consequently, the latent variable is part of the input mechanistic function. Since the hidden variable must model the causal relations of the latent variable, there must be edges from $HV$ to $Obs_{t+1}$.

    \item Additionally, the edges from $Obs_t$ to $Obs_{t+1}$ are kept since the presence of $HV$ does not replace the causal relations previously learned by the agent, but conditions them.
\end{enumerate}

\begin{figure}[ht]
\centering
\includegraphics[width=\textwidth]{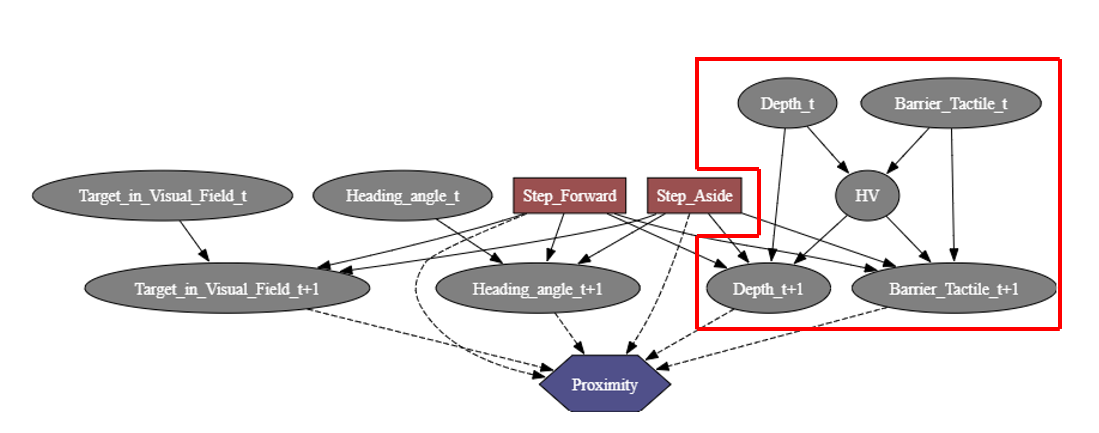}
\caption{DDN after introducing the new HV. The red box encloses the new learned subgraph. }
\label{fig:DN_with_HV}
\end{figure}

\subsection{CPT Estimation}
Next, the parameters of the CPTs, after introducing the hidden variable in the DDN´s graph, must be learned. That is, the hidden variable´s children and parents CPTs parameter estimation is performed by Hard Weighted Expectation Maximization. Specifically, after experiencing an epoch, the agent computes the probabilities of influence of the hidden variable at each time observation, as well as the difference in utility between the current and the previous observation. \\

This is followed by the application of Hard Weighted EM, which is iterated until convergence. In the initial iteration, the value of the hidden variable in each observation is imputed using the probability of influence, and each observation is assigned a weight based on the difference in utility $w_i$:
\begin{equation}
    \begin{split}
        &w_0=1,\\
        w_i= 1+ |U(\textbf{x}_{i-1})-&U(\textbf{x}_{i})| \;\; \forall i=1,\dots, t_{max}
    \end{split}
\end{equation}

The introduction of these $w_i$ serves to enhance the weighting of observations that exhibit minimal surprise but a significant impact on utility, while reducing the weighting of observations that exhibit a high degree of surprise but lack relevance to the observed differences in utility.
This approach enables the agent to learn a representation of HV that aligns with its existing knowledge, thereby enabling it to behave in a manner consistent with its previous behaviour in the absence of HV.
For the rest of the iterations, the weights associated with the observations are kept, while the value of the hidden value is imputed using the DDN and the parameters estimated in the previous iteration for  $P(HV=1)$.\\

\begin{algorithm}[ht]
    \SetAlgoLined
    \KwResult{CPT Estimation}
    \textbf{Inizialization:}\\
     Compute Probability of Influence of $HV$ ($IP_{HV}$) for each observation $i$ in dataset $D$ \;
     
     Compute weights $w_i$ \;
     
     \textit{Expectation $0$}: Build dataset $D^\prime_0$ from $D$, including imputed values of $HV$ using $IP_{HV}$ and weights $w_i$
     
     \textit{Maximization $0$} : Learn parameters $\Theta_0$ of $DDN_0$ by MLE or MAP\;
     
     \While{$|\Theta_{t} - \Theta_{t-1}| > \epsilon$}{
          \textit{Expectation-step:} Generate new dataset $D^\prime_t$ with imputed values for $HV$ using $DDN_{t-1}$\;
            
          \textit{Maximization-step:} Use $D^\prime_t$ for learning $\Theta_{t}$ with MLE or MAP\;
     } 
     \caption{CPT estimation with Hard Weighted EM}
     \label{alg:weighted EM}
\end{algorithm}

Returning to the Learning to Detour case, after learning, the hidden variable $HV$, a new transition must be learned so that expected observation $BT = 1$ is realized when the $\text{Step Forward}$ action is activated very close to the barrier. Also, after learning the hidden variable $HV$, a new transition must be learned so that the expected observation in $Depth$ is realized when the $\text{Step Forward}$ action is activated and the agent is close to the barrier. 
Lastly, when the agent is close to the end of the barrier, the expected observation of $Tactile = 0$ is realized, when $\text{Step Aside}$ is selected to the right.

\subsection{Learning to Detour behavior}

Thus, after learning the hidden variable $HV$, it can be observed from Figure \ref{fig:Actions_after_training}. that the agent's behaviour has undergone an important transformation compared to the behaviour before learning.

\begin{figure}[ht]
\centering
\includegraphics[width=0.7\textwidth]{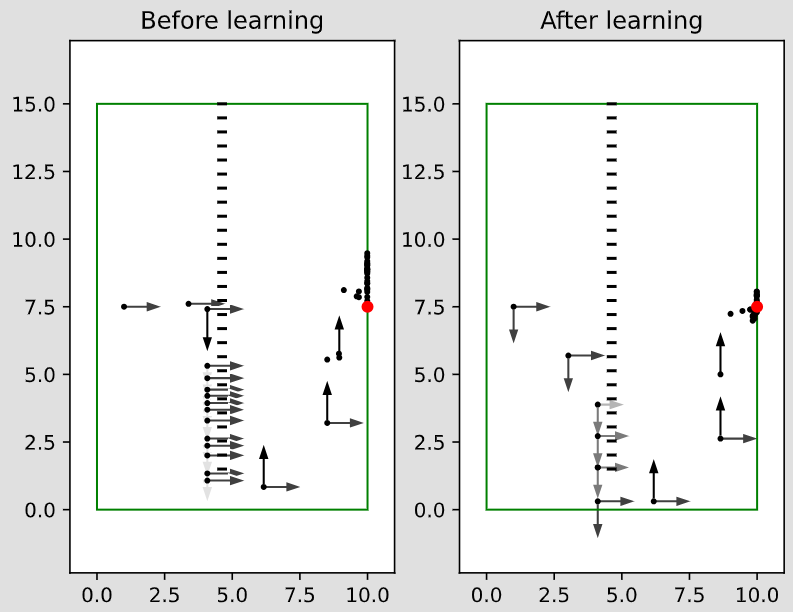}
\caption{Trajectories before and after learning $HV$}
\label{fig:Actions_after_training}
\end{figure}

The first thing to notice, is the shift from the agent's exclusive reliance on the $Step \; Forward$ action (before learning) to a more nuanced approach (during learning), involving a reduction in the power of $Step \; Forward$ decisions and the initiation of $Step \; Aside$ actions to the right. This is due to the changes made to the CPT of $BT_{t+1}$. As can be seen in the marginalized CPT of $BT_{t+1}$ before and after learning (Table \ref{tab:Table_BT_before and after}), there is a greater probability of a $BT=1$ in the presence of the $HV$. This CPT, along with the CPT learned from $HV$ (Table \ref{tab:CPT HV table}), enables the agent to act in this way.\\

\begin{table}[ht]
    \centering
    \caption{Marginalized CPT of $BT_{t+1}$ before and after learning.}
    \begin{tabular}{cccc}
        \cline{3-4}
                   &           & \multicolumn{2}{c}{$BT_{t+1}$} \\ \cline{3-4} 
        \multicolumn{2}{c}{$BT_t$} & $0$          & $1$         \\ \hline
        \multicolumn{2}{c}{0}  & $0.954$      & $0.046$     \\
        \multicolumn{2}{c}{1}  & $0.673$      & $0.327$     \\ \hline
    \end{tabular}
    \vfill
    \vspace{5mm}
    \begin{tabular}{cccccc}
        \cline{5-6}
                           &                   &            &           & \multicolumn{2}{c}{$BT_{t+1}$} \\ \cline{5-6} 
        \multicolumn{2}{c}{$HV$}                 & \multicolumn{2}{c}{$BT_{t}$} & 0          & 1         \\ \hline
        \multicolumn{2}{c}{\multirow{2}{*}{0}} & \multicolumn{2}{c}{0}  & $0.954$      & $0.046$     \\
        \multicolumn{2}{c}{}                   & \multicolumn{2}{c}{1}  & $0.673$      & $0.327$     \\
        \multicolumn{2}{c}{\multirow{2}{*}{1}} & \multicolumn{2}{c}{0}  & $0.862$      & $0.138$     \\
        \multicolumn{2}{c}{}                   & \multicolumn{2}{c}{1}  & $0.296$      & $0.704$     \\ \hline
    \end{tabular}
    \label{tab:Table_BT_before and after}
\end{table}

\begin{table}[ht]
\centering
\caption{Learned CPT of $HV$.}
\begin{tabular}{cccc}
\cline{3-4}
                   &       & \multicolumn{2}{c}{$HV$} \\ \cline{3-4} 
$BT_t$             & $D_t$ & \textbf{0}  & \textbf{1} \\ \hline
\multirow{5}{*}{0} & 0     & 0.815       & 0.185      \\
                   & 1     & 0.703       & 0.297      \\
                   & 2     & 0.250       & 0.750      \\
                   & 3     & 0.500       & 0.500      \\
                   & 4     & 0.500       & 0.500      \\
\multirow{5}{*}{1} & 0     & 0.500       & 0.500      \\
                   & 1     & 0.083       & 0.917      \\
                   & 2     & 0.111       & 0.889      \\
                   & 3     & 0.500       & 0.500      \\
                   & 4     & 0.500       & 0.500      \\ \hline
\end{tabular}
\label{tab:CPT HV table}
\end{table}
The second significant event occurs when the agent collides with the barrier. During learning,  the agent reduces the power of the $Step \; Forward$ action. This is a consequence of the agent's adaptation, which has led to a reduction in the expected utility of activating $Step \; Forward$ when the probability of $HV=1$ is high.
It is important to notice that at this moment, another important change takes place within our framework. That is, one of the key objectives is to enable the agent to learn from unexpected outcomes, which will enable it to make more accurate predictions and, therefore, avoid surprises in the future. In our case, Figures \ref{fig:Surprises_BT} and \ref{fig:Surprises_D} illustrate that the surprises associated with $BT$ (Barrier Tactile) and $D$ (Depth) are reduced in comparison to those observed prior to learning. This clearly indicates that the predictions made by the agent are more accurate with respect to the environment when the $HV$ factor is present.
In summary, before learning, the agent tends to go straight to the barrier and bumps into it in many  occasions. Yet, after learning, the agent tends to avoid the barrier by side stepping around it in its trajectory to the target. 

\begin{figure}[ht!]
\centering
\includegraphics[width=0.5\textwidth]{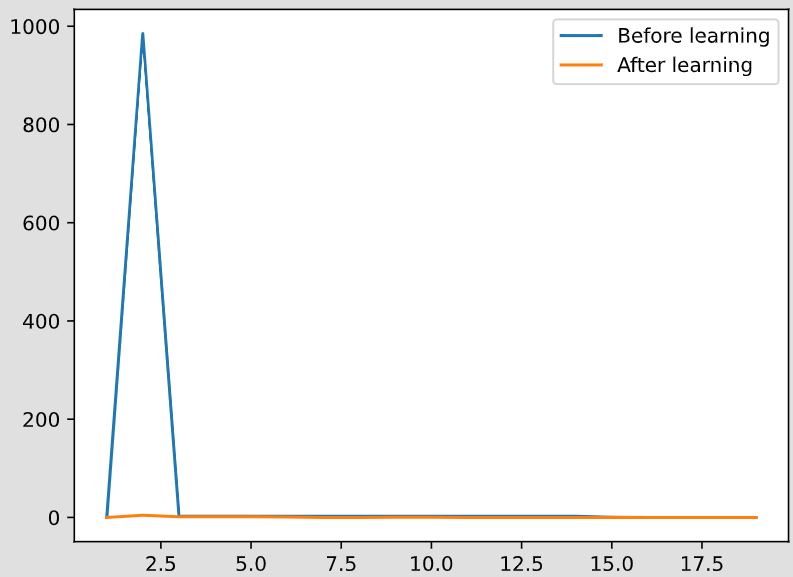}
\caption{Surprise coefficients over $BT$ (Barrier Tactile)  before and after learning $HV$.}
\label{fig:Surprises_BT}
\end{figure}

\begin{figure}[ht!]
\centering
\includegraphics[width=0.5\textwidth]{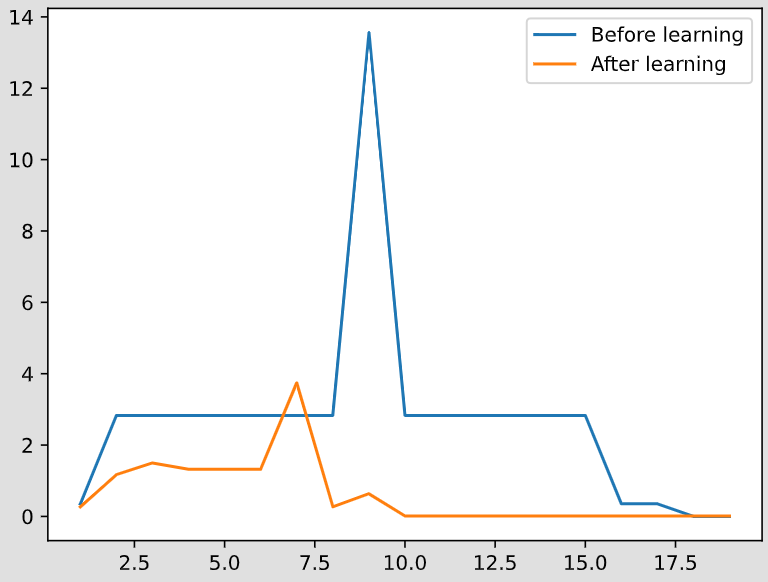}
\caption{Surprise coefficients over $D$ (Depth) before and after learning $HV$.}
\label{fig:Surprises_D}
\end{figure}

\begin{figure}[ht!]
\centering
\includegraphics[width=0.5\textwidth]{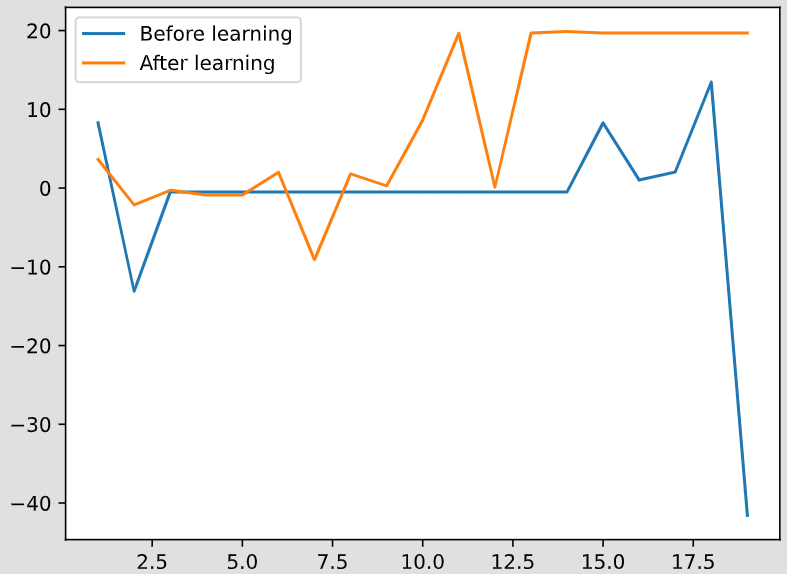}
\caption{Surprise coefficients over the Utility before and after learning $HV$.}
\label{fig:Surprises_Utility}
\end{figure}

Finally, after overcoming the obstacle, the agent makes essentially the same actions in both scenarios, demonstrating that the incorporation of the $HV$ and the acquisition of CPTs did not impede the agent's appropriate behaviour when the probability of $HV = 1$ is low, allowing the agent to successfully reach the target.

\section{Discussion}
This is also a first step towards a theory of Active Structure Learning for Resilient autonomous agents and robots. They must be able to construct causal internal models (Bongard et al.; Corbacho et al)\cite{Bongard2006,Corbacho2005}. As Bongard et al. \cite{Bongard2006} express, animals sustain the ability to operate after injury by creating qualitatively different compensatory behaviors. Although such robustness would be desirable in engineered systems, most machines fail in the face of unexpected damage. Corbacho et al.\cite{Corbacho2005} is one of the first complete models of animal behavior learning after a lesion (that impairs the animal´s motor capabilities)  by internal causal model construction. The research here presented is also related to work on lifelong machine learning (Chen and Liu)\cite{Chen2018}. That is, a machine learning paradigm that continuously learns by accumulating past knowledge that it then uses in future learning and problem solving. It is also related to the field of probabilistic robotics, which is concerned with perception and control in the face of uncertainty to endow robots with a new level of robustness in real-world situations \cite{Thrun2005}. It is also related to the field of autonomous mobile robot navigation (Arkin, 1998, 2005; Bekey, 2005; Pandey et al., 2017)\cite{Arkin2005,Arkin1998a,Bekey2005,Pandey2017}. 

\section{Work in progress}
\label{sect:Work in progress}
In this section we briefly describe work in progress that we do not include in this paper, due to space limitations, or level of maturity of the current development.

\subsection{Other simulations and possible real world applications}
One of the limiting aspects of the work presented in this specific paper, is that we have only included a specific example within the field of simulated robotics. Nevertheless, we have already started to work with other more realistic examples that will give us the opportunity to discover aspects where our framework can be improved.
The first extension, we are already  working on, consists on a robotic simulation of a Kephera robot, first within the Webots simulator, to later transfer it to the real physical robot. 
This work is a continuation of previous work on robotic machine learning implementations \cite{Lago00,Lago01,Weitzenfeld98}, but now with a much more robust, probabilistic and complete causal structure learning framework.
In this context, we have found that sensors and actuators are sometimes very noisy, leading to cases where the observed value does not match what should have been observed, which could lead to surprises caused by situations where actually there is no latent variable.  Within this example we also wish that the observation and action variables are vector random variables, not just scalar random variables, which would lead us to extend our framework to deal with this more general case. 
To solve the first problem, we look at the direct use of the surprise divergence as a surprise coefficient, so it could be used a probability distribution, that is not a degenerate one, over possible outcomes of the observation variable conditioned on the value observed. For the second extension, we are working on extending the surprise divergence to vectorial spaces, looking at the similarities of the divergence with the Mahalanobis distance.\\

Another example we have started working on is a practical example within the medical field. In this case, we have already developed a  digital twin of the human respiratory system, where we can simulate patients with different diseases, and our framework aims to capture the causal relationships and latent variables that are simulated within the digital twin. Some of the problems we face in this situation are dealing with multiple latent variables that may be related to each other or have relationships in common with the observed variables, in addition to dealing with continuous variables. In order to cope with this, we are strengthening  the framework to be able to deal with cases where there are several surprises in the utility, yet not all the observation variables that undergo a surprise coincide, but only some, which would be the parents in common that would have the hidden variables. We also hope that this example, in the future, can be related to Game Theory, since the system could be considered as a non-cooperative game.  It would be necessary to give the digital twin some kind of utility that would make it more difficult for the other party to guess its structure. Following this path, we could extend the framework to act in situations where the causal structure of all variables, observed or unobserved, may change over time.

\subsection{Extending the theoretical framework of the Theory of surprise}
We also want to make our theory of surprise more robust, and to that end, we need to continue to work on generalizing it to make it more versatile.
Apart from the cases mentioned in the previous section, which entail their respective extensions of the framework, we are working on finding out how the entropy and information dispersion estimators tend to be distributed according to the sample, and whether these distributions depend on the cardinality. If this is the case, we would try to add a normalization factor to the surprise divergence, which would allow us to compare distributions that do not have the same cardinality, since this is currently possible with the surprise divergence as it is, but it is very uninformative since, like the Kullback-Leibler divergence, the result is $\infty$. To deal with this, we try to solve it from Bayesian statistics, and assume that the parameter estimators come from samples of a Dirichlet variable with a uniform distribution over the parameters.

\subsection{Exploration behaviour after surprise event}
We are also working on an exploration algorithm that is activated each time a latent variable is detected. This would work in such a way that, when a latent variable is detected, we first observe the tuple with the values of the observation variables that have been observed at that time, and make a change in the utility function that prioritizes the agent to visit the tuples of values that are probabilistically close to the tuple of values that have been observed. If, when exploring these nearby states, there are also surprises observed with respect to the original utility function, then the tuples of nearby values that the agent would have to explore, would be added to a queue.
In this way, we would avoid having the agent repeat the epochs from the beginning in order to learn, and the learning process would be done in a single epoch, but with a greater number of iterations.
This idea is quite related to the concept of curiosity \cite{SCHMIDHUBER2009,sontakke2021}, so this could be an opportunity to formally relate the concept of surprise and curiosity of an agent in a future.
\section{Future Work}
The framework proposed in this paper allows for multiple generalizations and extensions. For completeness reasons, we list here some of them that we intend to address in future work.
\begin{itemize}
    \item A comparative study of ACSLWL with some existing causal structural learning methods. Although we have seen that the framework gives good results in the designed environments, it remains to be compared whether the results with respect to other causal structural learning algorithms that also deal with latent variables. Before making the comparison, it would first be necessary to do tests in which causal and latent variable learning are done in a single process, unlike in this work where we have focused on latent variable learning for simplicity purposes.
    \vspace{2mm}
    
    \item One potential issue that may arise from dealing with latent variables is that the complexity of the algorithms can grow and make them infeasible in real-world scenarios. Consequently, doing a complexity analysis, both at the computational and memory level, would be a good next step to make the proposed algorithms more robust.
    \vspace{2mm}
    
    \item After the complexity analysis, the next step would be to optimize the algorithms we have presented here to make them more efficient. We are beginning to make online versions of the algorithms, as this would allow the algorithms to be quite more efficient and consistent, as well as bringing us closer to making online and automatic discovery of latent variables possible \cite{chen2022automated}. 
    In this regard, we are working on online algorithms for estimating the entropy of a random variable and the transfer entropy of a process from the data observed by the agent. These algorithms are based on  reformulations of the entropy formula, and we have yet to demonstrate the differences in complexity with respect to traditional estimation algorithms. Within this subsection, the development of the exploration algorithm (described in Work in progress) would also be included, as it would be necessary to prove that it is indeed more efficient than the algorithm we have described in this paper.
    \vspace{2mm}
    
    \item Throughout the framework, we have actually assumed that there is only one latent variable affecting the agent at any given point in time. In the real world, this might not be true, so we have yet to generalize our framework to deal with cases where there is more than one variable. Also, we have assumed that the latent variables are binary, and this is not usually the case, so our framework should also be able to work with latent variables of higher cardinalities.
    We propose  that the distinction of latent variables will depend on the observed variables receiving surprises. That is, two latent variables will be distinct if they do not completely coincide in the observed variables with surprises.\\
    Regarding the cardinality of latent variables, we believe that a good first step would be to reason about the interpretation of the meaning of a latent variable with more than two values. On the one hand, we could understand that the cardinality of latent variables depends on the cardinalities of the related observed variables, and on what is the product of a ``clustering'' of states. On the other hand, we could understand that the cardinality of latent variables depends on the causal relationships themselves, and that a latent variable itself is a collection of different (binary) latent variables that behave in a very similar way, or that share causal relationships with some very important observed variables, and therefore can be understood as a latent variable that has variants or different behaviors.
    \vspace{2mm}
    
    \item Another limitation is the assumption that the observation and action variables are discrete. The reason for working with this assumption is that, although a priori, it does not present a problem at the theoretical level, at the practical level there are usually many difficulties when using decision networks with continuous observation or action variables. Therefore, it remains to give a formal generalization to include continuous variables.
    \vspace{2mm}
    
    \item Finally, in some cases, the impact of latent variables is observed during periods of relevant surprises rather than surprises at a specific time. It may also be the case that the surprise in the observation variables is observed with a time lag with respect to the instant of the influence of the latent variable. Thus, it remains to generalize the theoretical framework to account for all these possibilities.
    
\end{itemize}


\begin{credits}
\subsubsection{\ackname} 
Fernando Corbacho is very grateful to Michael A. Arbib and Christoph von der Malsburg for very enlightening discussions on adaptive biological agents. Wei Min Shen for autonomous agents and robots.  George A. Bekey for stimulating discussions on autonomy. Alfredo Weitzenfeld and Ronald C. Arkin´s groups for implementation ideas in Robotics.

\subsubsection{\discintname}
The authors have no competing interests to declare that are
relevant to the content of this article.
\end{credits}

%
%
%


\appendix
\section{Divergence surprise proofs}\label{Appendix divergence proofs}

\begin{prop}\label{proof prop1}
$D_s(Q||P)=0 \Leftrightarrow Q=P$.
\end{prop}
\begin{proof}
    Assume that $D_S(Q||P)=0$, then
    \[0=D_S(Q||P)=\frac{H(Q,P)-H(P)}{\sqrt{V_I(P)}} \Leftrightarrow H(Q,P)=H(P)\Leftrightarrow\]
    \[\Leftrightarrow Q=P.\]
\end{proof}

\begin{thm}\label{proof theorem1}
    $D_{KL}(Q||P) \xrightarrow{} 0$ $\Leftrightarrow$ $D_S(Q||P)^2 \xrightarrow{} 0$.
\end{thm}
\begin{proof}
    \[D_S(Q||P)^2= \frac{1}{V_I(P)}\left(D_{KL}(Q||P) + H(Q)-H(P)\right)^2=\]
    \[=\frac{D_{KL}(Q||P)^2}{V_I(P)} +\frac{2D_{KL}(Q||P)}{V_I(P)}(H(Q)-H(P)) + \frac{(H(Q)-H(P))^2}{V_I(P)}=\]
    \[=A \; + \; B \; + \; C.\]
    
    \begin{itemize}
        \item Firstly, if $D_{KL}(Q||P)\rightarrow 0 \; \Rightarrow \; A\rightarrow 0$.
        
        \item Furthermore, if $D_{KL}(Q||P)\rightarrow 0 \; \Rightarrow \;  q_i-p_i\rightarrow 0 \; \forall i$, which implies that $H(Q)-H(P)\rightarrow 0$. Therefore, $C \xrightarrow{}0$.
        
        \item Finally, in $B$ it can be shown that $\frac{2D_{KL}(Q||P)}{V_I(P)} \xrightarrow{}0$ multiplies to something that is bounded, $|H(Q)-H(P)|\leq \max\{H(Q),H(P)\}\leq log(n)$, hence $B\xrightarrow{} 0$.
     
    \end{itemize}

    With regard to the opposite implication, from proposition \ref{prop:D_s_tend_0} it follows that $D_S(Q||P)^2 \rightarrow 0 \; 
    \Rightarrow \; q_i-p_i\rightarrow 0 \; \forall i$, thus $D_{KL}(Q||P) \rightarrow 0$.
    
\end{proof}

\begin{thm}
    \label{proof theorem2}
     Let $\{X_1,\dots, X_n\}$ be a sample of independent discrete random variables, each one with the same probability distribution $P=[p_1,\dots,p_k]$, then
     \[\sqrt{n} \cdot D_S(\hat{P}||P) \xrightarrow{d} N(0,1),\] where $\hat{P}=[\hat{p}_1,\dots,\hat{p}_n]$ is the maximum likelihood estimator of $P$ from $\{X_1,\dots, X_n\}$.  
\end{thm}

\begin{proof}
Given the sample $\textbf{X}=\{X_1,\dots,X_n\}$, the sample formed by $\textbf{W}=W_1,\dots,W_n\}$, where $W_i=-log(P(X=X_i))$, is also a sample of independent discrete random variables, which follow the probability distribution $P$.\\

Therefore:
\[\mu_\textbf{W}=E_P[-log(P(X))]=H(P), \]
\[\sigma^2_\textbf{W}=E_P[(-log(P(X))-H(P))^2]=V_I(P),\]
\[ S_n=\sum_{i=1}^n W_i=\sum_{i=1}^n -log(P(X=X_i)=-n\sum_{j=1}^k \hat{p}_jlog(p_j)=n\cdot H(\hat{P},P).\]

Using the Central Limit Theorem, it follows that
\[\frac{S_n - n\mu_\textbf{W}}{\sqrt{n\sigma_\textbf{W}^2}} \xrightarrow{d} N(0,1),\]
that is to say,
\[\frac{S_n - n\mu_\textbf{W}}{\sqrt{n\sigma_\textbf{W}^2}}=\frac{n}{\sqrt{n}}\frac{H(\hat{P},P) - H(P)}{\sqrt{V_I(P)}}=\sqrt{n}\cdot D_S(\hat{P}||P) \xrightarrow{d}N(0,1).\]
\end{proof}

\begin{prop}\label{proof prop2}
    Let $\{X_1,\dots, X_n\}$ be a sample of independent random variables, each one with the same probability distribution $P=[p_1,\dots,p_k]$ and 
    \[\hat{V}_I(P)=\frac{1}{n-1}\sum_{i=1}^n (-log(P(X=X_i))-H(\hat{P},P)^2),\] then:
    \[\hat{V}_I(P) \xrightarrow{P} V_I(P).\]
\end{prop}

\begin{proof}
    \[\hat{V}_I(P)=\frac{1}{n-1}\sum_{i=1}^n \left[-log(P(X=X_i))-H(\hat{P},P)\right]^2=\]
    \[=\frac{1}{n-1}\sum_{i=1}^n \left[-log(P(X=X_i))-H(P) + H(P)-\frac{1}{n}\sum_{j=1}^n -log(P(X=X_j))\right]^2=\]
    \[=\frac{1}{n-1}\sum_{i=1}^n \left[-log(P(X=X_i))-H(P)-\frac{1}{n}\left(\sum_{j=1}^n -log(P(X=X_j) -H(P) \right)\right]^2.\]
    Labeling $D_i=-log(P(X=X_i) -H(P)$ y $\bar D=\frac{1}{n}\sum_i D_i$, it follows that
    \[\hat{V}_I(P)=\frac{1}{n-1}\sum_{i=1}^n \left[D_i -\bar D\right]^2=\frac{1}{n-1}\left[\sum_{i=1}^n D_i^2 +n\bar D^2 -2\bar D\sum_i^n D_i\right]=\]
    \[=\frac{n}{n-1}\left[\frac{1}{n}\sum_{i=1}^n D_i^2 -\bar D^2 \right].\]

    \noindent On one hand, using the Law of Large Numbers:
    \[\frac{1}{n}\sum_{i=1}^n D_i^2 \xrightarrow{P} E_P[D_i^2]=E_P[(-log(P(X=X_i) -H(P))^2]=V_I(P).\]
    On the other hand, also using the Law of Large Numbers,
    \[\bar D\xrightarrow{P}E_P[-log(P(X=X_i) -H(P)]=0.\]
    Being a convergence in probability to a degenerate distribution, it is also held that $\bar D \xrightarrow{d} 0$, and the Continuous Application Theorem can be applied with the function $f(x)=x^2$, obtaining:
    \[\bar D^2 \xrightarrow{d} 0  \longleftrightarrow \bar D^2 \xrightarrow{P} 0.\]
    \noindent This concludes the demonstration:
    \[\frac{n-1}{n} \hat{V}_I(P) = \frac{1}{n}\sum_{i=1}^n D_i^2 -\bar D^2 \xrightarrow{P} V_I(P).\]
    And since $\frac{n-1}{n} \xrightarrow{n\rightarrow \infty} 1$, then:
    \[\hat{V}_I(P) \xrightarrow{P} V_I(P).\]
\end{proof}

\begin{thm}
    \label{proof theorem3}
    Given a sample $\textbf{X}$ as in theorem \ref{theorem:Div_normal}, then
    \[\sqrt{n} \cdot \frac{H(\hat{P},P) - H(P)}{\sqrt{\hat{V}_I(P)}} \xrightarrow{d} N(0,1),\]
\end{thm}
\begin{proof}
    By proposition \ref{prop:prob_tend_info_disp} it follows that:
    \[\frac{V_I(P)}{\hat{V}_I(P)} \xrightarrow{P} 1.\]
    While, by theorem \ref{theorem:Expect_max_Divergencia_S}, it holds that:
    \[\sqrt{n} \cdot \frac{H(\hat{P},P) - H(P)}{\sqrt{V_I(P)}} \xrightarrow{d} N(0,1).\]
    Therefore, Slutsky`s theorem can be used since
    \[\sqrt{n} \cdot \frac{H(\hat{P},P) - H(P)}{\sqrt{\hat{V}_I(P)}}=\frac{V_I(P)}{\hat{V}_I(P)} \cdot \sqrt{n} \cdot \frac{H(\hat{P},P) - H(P)}{\sqrt{V_I(P)}} \xrightarrow{d} N(0,1).\]
\end{proof}

\section{Technical Appendix}
\label{appendix:Technical Appendix}
\subsection{Dealing with discretizations }

\subsubsection{Actions Space}
The main reason for discretising the domain of the action variables is to enable the agent to make decisions using a DDN, which will be then transformed into continuous values. In tables \ref{tab:Clase_Int_SF} and \ref{tab:Clase_Int_SA} we can see the correspondences between the categories of decisions and the intervals of values to which they correspond. 

\begin{table}[ht]
\centering
\caption{Correspondence between category and interval of $Step\; Forward$ discretization}
\begin{tabular}{ccc}
\cline{1-2}
 Category & Interval &  \\ \cline{1-2}
 0 & [0,0.5) &  \\ 
 1 & [0.5,1) &  \\ 
 2 & [1,1.5) &  \\ 
 3 & [1.5,2) &  \\ 
 4 & [2,2.5] &  \\ \cline{1-2}
\end{tabular}
\label{tab:Clase_Int_SF}
\end{table}

\begin{table}[ht]
\centering
\caption{Correspondence between category and interval of $Step\; Aside$ discretization}
\begin{tabular}{ccc}
\cline{1-2}
 Category & Interval &  \\ \cline{1-2}
 0 & [-2.5,-2.1) &  \\ 
 1 & [-2.1,-1.6) &  \\ 
 2 & [-1.6,-1.1) &  \\ 
 3 & [-1.1,-0.7) &  \\ 
 4 & [-0.7,-0.2) &  \\ 
 5 & [-0.2,0.2) &  \\ 
 6 & [0.2,0.7) &  \\ 
 7 & [0.7,1.1) &  \\ 
 8 & [1.1,1.6) &  \\ 
 9 & [1.6,2.1) &  \\ 
 10 & [2.1,2.5] &  \\ \cline{1-2}
\end{tabular}
\label{tab:Clase_Int_SA}
\end{table}

To go from discrete to continuous values it has been decided to select a random decision value using a uniform distribution with the appropriate bounds, depending on the discrete decision selected by the agent. For instance, if the agent decide to do $Step \; Forward $ with a value $1$, then it will returned to the Environment to execute $SA(x,y,\alpha,s)$ with $s\sim U(0.5,1)$. 

\subsubsection{Observation Space}
In the Observation Space, $Depth$ and $Heading \; Angle$ are the only observation variables that need a discretization.\\

For $Depth$, the interval $[0, d_{max} ]$ has been divided into five equal parts, where $d_{max}=\sqrt{x_{max}^2 + y_{max}^2}$

With regard to the $Heading \; Angle$, the interval $[-\pi, \pi]$ has been divided into nine equal parts.

\subsection{How the Agent and Environment interacts: Restrictors}

The environment will have two restrictors: $MapBounds$ and $BarrierImpact$.
\subsubsection{MapBounds}
The first restrictor ensures that the agent cannot exceed the limits set by the environment, that said, the agent is unable to perform an action that would result in its displacement outside of the rectangle $ABCD$, with $A = (0, 0)$, $B = (10, 0)$, $C = (10, 15)$, $D = (0, 15)$.\\

In the case that the agent performs an action that can take him out of the environment, that specific action will not be carried out.

\subsubsection{BarrierImpact}

On the other hand, the $BarrierImpact$ restrictor works in a similar way to the previous one, but with some nuances. This restrictor will detect when an action performed by the agent causes it to go beyond the barrier, impacting with any of its spikes. In this instance, rather than denying the action and preventing its execution, the value of the action will be substituted with the maximum value that does not activate the restrictor.\\

To illustrate this concept (Figure \ref{fig:BarrierImpact}), consider a scenario in which an agent decides to move forward five steps, but the barrier is three steps away in the same direction. In this case, the action of moving forward five steps will be modified by the action of moving forward three steps.\\
\begin{figure}[ht]
\centering
\includegraphics[width=0.35\textwidth]{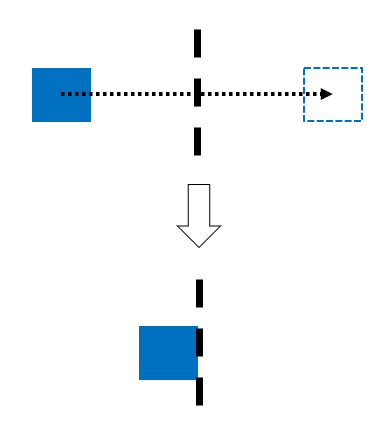}
\caption{$BarrierImpact$ restrictor}
\label{fig:BarrierImpact}
\end{figure}

\subsection{Transition Probabilities}
This section outlines the transition probability matrices employed by the agent. Each matrix consists of rows representing the values of $Obs_t$ and columns representing the values of $Obs_{t+1}$. As each probability transition $Obs_{t+1}$ may depends on the value of $Step \; Forward$ and $Step \; Aside$, the matrices are expressed in terms of the value of $SF$ and $SA$.

\subsubsection{$Depth$}
\[\begin{bmatrix}
0.9999 & \frac{SF}{6} & 0 & 0 & 0 \\
0.0001 & 1-\frac{SF}{6} & \frac{SF}{6} & 0 & 0 \\
0 & 0 & 1-\frac{SF}{6} & \frac{SF}{6} & 0 \\
0 & 0 & 0 & 1-\frac{SF}{6} & \frac{SF}{6} \\
0 & 0 & 0 & 0 & 1-\frac{SF}{6} 
\end{bmatrix}  \]

\subsubsection{$Heading \;Angle$}
\setcounter{MaxMatrixCols}{20}
\[\begin{bmatrix}
F_1 & F_1^- & 0 & 0 & 0 & 0 & 0 & 0 & 0 & 0 & F_4^+  \\
F_1^+ & F_1 & F_1^- & 0 & 0 & 0 & 0 & 0 & 0 & 0 & 0 \\
0 & F_1^+ & F_1 & F_2^- & 0 & 0 & 0 & 0 & 0 & 0 & 0 \\
0 & 0 & F_1^+ & F_1 & F_2^- & 0 & 0 & 0 & 0 & 0 & 0 \\
0 & 0 & 0 & F_2^+ & F_1 & F_2^- & 0 & 0 & 0 & 0 & 0 \\
0 & 0 & 0 & 0 & F_2^+ & F_1 & F_3^- & 0 & 0 & 0 & 0 \\
0 & 0 & 0 & 0 & 0 & F_2^+ & F_1 & F_3^- & 0 & 0 & 0 \\
0 & 0 & 0 & 0 & 0 & 0 & F_3^+ & F_1 & F_4^- & 0 & 0 \\
0 & 0 & 0 & 0 & 0 & 0 & 0 & F_3^+ & F_1 & F_4^- & 0 \\
0 & 0 & 0 & 0 & 0 & 0 & 0 & 0 & F_4^+ & F_1 & F_4^- \\
F_1^- & 0 & 0 & 0 & 0 & 0 & 0 & 0 & 0 & F_4^+ & F_1 
\end{bmatrix}\]
Where each $F_i$ function is defined as follows:
\[F_1(SF,SA)=p\cdot\left(1-\frac{|SA-5|}{5}\right) + (1-p)\cdot\left(1-\frac{SF}{6}\right),\]
\[F_1^-(SF,SA)=p\cdot max\left(0,\frac{SA-5}{5}\right) + (1-p)\cdot\frac{SF}{6},\]
\[F_1^+(SF,SA)=-p\cdot min\left(0,\frac{SA-5}{5}\right),\]
\[F_2^-(SF,SA)=-p\cdot min\left(0,\frac{SA-5}{5}\right) + (1-p)\cdot\frac{SF}{6},\]
\[F_2^+(SF,SA)=p\cdot max\left(0,\frac{SA-5}{5}\right),\]
\[F_3^-(SF,SA)=-p\cdot min\left(0,\frac{SA-5}{5}\right),\]
\[F_3^+(SF,SA)=p\cdot max\left(0,\frac{SA-5}{5}\right) + (1-p)\cdot\frac{SF}{6},\]
\[F_4^-(SF,SA)=p\cdot max\left(0,\frac{SA-5}{5}\right), \]
\[F_4^+(SF,SA)=-p\cdot min\left(0,\frac{SA-5}{5}\right) + (1-p)\cdot\frac{SF}{6}.\]
The parameter $p$  serves to balance the importance of $SA$ and $SF$ in the transition probability. After testing that the agent behaved correctly in different scenarios the parameter was set to $p=0.2$   
\subsubsection{$Target \; in \; Visual \; Field$}
\[    \begin{bmatrix}
    0.9 & 0.4\cdot\left(\frac{|sa-5|}{5}\right) + 0.01\cdot\left(1- \frac{|sa-5|}{5}\right) \\
    0.1 & 0.6\cdot\left(\frac{|sa-5|}{5}\right) + 0.99\cdot\left(1- \frac{|sa-5|}{5}\right)
    \end{bmatrix}
\]

\subsubsection{$Barrier \; Tactile$}
\[ \begin{bmatrix}
C(SF,SA) & L(SF,SA) \\
1-C(SF,SA) & 1-L(SF,SA) 
\end{bmatrix}, \]

\[L(SF,SA)= \left\{ \begin{array}{lcc}
0.5 & si & SA<5  \\ 
0.6 & si & SA=5\\
0.65 & si & SA>5 \end{array} \right.\]

The objective of the $L$ function is to introduce rightward laterality to the agent when it hits the barrier and to prevent the agent from ending the epochs without reaching the target.

\[C(SF,SA)= \left\{ \begin{array}{lcc} 
0.99 & if & SA=5\\
0.95 & otherwise &  \end{array} \right.\]

On the other hand, the objective of $C$ is to encourage the agent not to move sideways if it is not necessary.

\end{document}